\documentclass[sigconf]{acmart}

\makeatletter
\def\@ACM@checkaffil{
    \if@ACM@instpresent\else
    \ClassWarningNoLine{\@classname}{No institution present for an affiliation}%
    \fi
    \if@ACM@citypresent\else
    \ClassWarningNoLine{\@classname}{No city present for an affiliation}%
    \fi
    \if@ACM@countrypresent\else
        \ClassWarningNoLine{\@classname}{No country present for an affiliation}%
    \fi
}
\makeatother

\usepackage[ruled,linesnumbered]{algorithm2e}
\usepackage{enumitem}
\usepackage{booktabs}
\usepackage{xcolor}
\usepackage{caption}
\usepackage{subcaption}
\usepackage{amsmath}
\usepackage{multirow}
\usepackage{amsfonts}
\usepackage{color}
\usepackage{colortbl}
\definecolor{mygray}{gray}{.9}
\definecolor{LightCyan}{rgb}{0.88,1,1}
\newcommand{\nop}[1]{}

\newcommand{\model}{{MP-KD}\xspace }

\AtBeginDocument{%
  \providecommand\BibTeX{{%
    \normalfont B\kern-0.5em{\scshape i\kern-0.25em b}\kern-0.8em\TeX}}}

\copyrightyear{2023} 
\acmYear{2023} 
\setcopyright{acmlicensed}\acmConference[WWW '23]{Proceedings of the ACM Web Conference 2023}{April 30-May 4, 2023}{Austin, TX, USA}
\acmBooktitle{Proceedings of the ACM Web Conference 2023 (WWW '23), April 30-May 4, 2023, Austin, TX, USA}
\acmPrice{15.00}
\acmDOI{10.1145/3543507.3583407}
\acmISBN{978-1-4503-9416-1/23/04}

\begin{document}
\title{Mutually-paced Knowledge Distillation for Cross-lingual Temporal Knowledge Graph Reasoning}

\author{Ruijie Wang, Zheng Li$^{2\dagger}$, Jingfeng Yang$^{2}$, Tianyu Cao$^{2}$, Chao Zhang$^{3}$, \\Bing Yin$^{2}$ and Tarek Abdelzaher$^{1\dagger}$}
\thanks{$^{\dagger}$Corresponding authors in UIUC and Amazon.com Inc}
\affiliation{%
  \institution{$^{1}$Department of Computer Science, University of Illinois Urbana-Champaign,~$^{2}$Amazon.com Inc}
  \institution{$^{3}$School of Computational Science and Engineering, Georgia Institute of Technology}
}

\email{{ruijiew2, zaher}@illinois.edu,\space chaozhang@gatech.edu}
\email{{amzzhe, jingfe, caoty, alexbyin}@amazon.com}



\renewcommand{\shortauthors}{Wang, et al.}

%

\newcommand{\fix}{\marginpar{FIX}}
\newcommand{\new}{\marginpar{NEW}}

\begin{abstract}
This paper investigates cross-lingual temporal knowledge graph reasoning problem, which aims to facilitate reasoning on Temporal Knowledge Graphs (TKGs) in low-resource languages by transfering knowledge from TKGs in high-resource ones. The cross-lingual distillation ability across TKGs becomes increasingly crucial, in light of the unsatisfying performance of existing reasoning methods on those severely incomplete TKGs, especially in low-resource languages. However, it poses tremendous challenges in two aspects. First, the cross-lingual alignments, which serve as bridges for knowledge transfer, are usually too scarce to transfer sufficient knowledge between two TKGs. Second, temporal knowledge discrepancy of the aligned entities, especially when alignments are unreliable, can mislead the knowledge distillation process. We correspondingly propose a mutually-paced knowledge distillation model \model, where a teacher network trained on a source TKG can guide the training of a student network on target TKGs with an alignment module. Concretely, to deal with the scarcity issue, \model generates pseudo alignments between TKGs based on the temporal information extracted by our representation module. To maximize the efficacy of knowledge transfer and control the noise caused by the temporal knowledge discrepancy, we enhance \model with a temporal cross-lingual attention mechanism to dynamically estimate the alignment strength. The two procedures are mutually paced along with model training. Extensive experiments on twelve cross-lingual TKG transfer tasks in the EventKG benchmark demonstrate the effectiveness of the proposed \model method.

\end{abstract}
\maketitle
\section{introduction}


Temporal Knowledge Graphs (TKGs)~\cite{YAGO,ICEWS18,WIKI,acekg} characterize temporally evolving events, where each event, represented as ({\em subject}, {\em relation}, {\em object}), is associated with temporal information ({\em time}), e.g., ({\em Macron}, {\em reelected}, {\em French president}, {\em 2022}). TKGs has facilitated various knowledge-intensive Web applications with timeliness, such as question answering~\cite{KBQA}, product recommendation~\cite{RippleNet,TKG4Rec,TKG4Rec2,RETE}, and social event forecasting~\cite{KG4Social,DyDiff-VAE,andgan,belief,misinfo,polarization}. 

As new events are continually emerging, modern TKGs are still far from being complete. Conventionally, the TKG construction process relies primarily on information extraction from unstructured corpus~\cite{WIKI,YAGO, EventKG}, which necessitates extensive manual annotations to keep up with changing events. For instance, the recent transition from Trump to Biden as the President of the United States has not been reflected in many TKGs, highlighting the need for timely updates. This spurs research on temporal knowledge graph reasoning to automate evolving events prediction over time~\cite{TA-DistMult,Know-Evolve,Renet,RE-GCN}. Unfortunately, the problem of TKG incompleteness is particularly pronounced in low-resource languages, where it is unable to collect enough corpus and annotations to support robust TKG construction. This results in suboptimal reasoning performance and distinctly unsatisfying accuracy in predicting recent and future events.


\begin{figure}
    \centering
    \includegraphics[width = 1.0\linewidth]{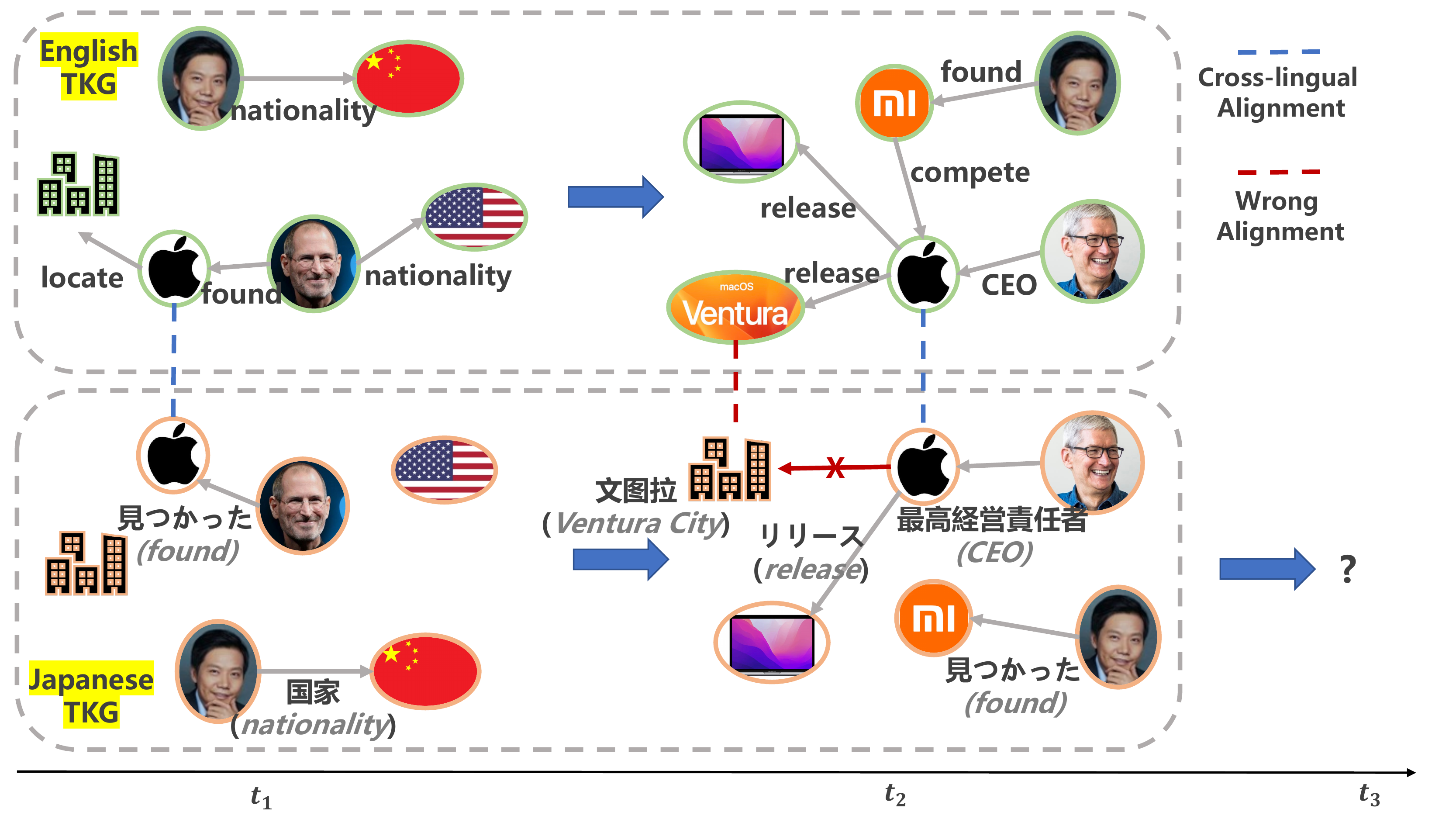}
    \caption{An illustrative example of cross-lingual reasoning on TKGs. 1) We aim to transfer knowledge from English TKG to Japanese TKG, where the English version provides more complete information; 2) Cross-lingual alignments only cover a small ratio of entities, e.g., Apple Inc; 3) Cross-lingual alignments can be noisy and misleading, e.g., A city called Ventura is linked to new macOS Ventura at $t_2$, introducing noise for reasoning in Japanese.}
    \label{fig:illustration}
\end{figure}

Inspired by the incompleteness issue facing low-resource languages in constructing TKGs, we introduce a novel task named Cross-Lingual Temporal Knowledge Graph Reasoning (as shown in Figure~\ref{fig:illustration}). This task aims to alleviate the reliance on supervision for TKGs in low-resource languages (referred to as the target language) by transferring temporal knowledge from high-resource languages (referred to as the source language)~\footnote{In this paper, for the sake of brevity, we interchangeably use the terms high-resource/low-resource and source/target.}. In contrast, all the existing efforts are either limited to reasoning in monolingual TKGs (usually high-resource languages, e.g., English)~\cite{TA-DistMult,Know-Evolve,Renet,RE-GCN}, or multilingual static KGs~\cite{KEnS,AlignKGC,SS-AGA}. To the best of our knowledge, cross-lingual TKG reasoning that transfers temporal knowledge between TKGs has not been investigated. 


The fulfillment of this task poses tremendous challenges in two aspects: 1) \textbf{Scarcity of cross-lingual alignment}: as the informative bridge of two separate TKGs, cross-lingual alignment is imperative for cross-lingual knowledge transfer~\cite{AlignKGC,KEnS,SS-AGA}. However, obtaining alignments between languages is a time-consuming and resource-intensive process that heavily relies on human annotations. The transfer of knowledge through a limited number of alignments is often insufficient to fully enhance the TKG in the target language. 2) \textbf{Temporal knowledge discrepancy}: the information associated with two aligned entities is not necessarily identical, especially with regards to temporal patterns. Utilizing a rough approach to equate the aligned entities at all times can result in the transfer of misleading knowledge and negatively impact performance. This becomes more pronounced when the alignments are noisy and unreliable. For example, at the time step $t_2$, a new event about operating system ``{\it Ventura}'' from Apple company occurs in the source English TKG, and meanwhile there is a noisy aligned entity ``{\it Ventura city}'' in the target Japanese TKG. Directly pulling those two entities at this point, can inevitably introduce  noise and fail to predict a set of related events in the target TKG. Therefore, it is crucial to dynamically regulate the alignment strength of each local graph structure over time in order to maximize the effectiveness of cross-lingual knowledge distillation.

In this paper, we propose a novel Mutually-paced Knowledge Distillation (\model) framework, where a teacher network learns more enriched temporal knowledge and reasoning skills from the source TKG to facilitate the learning of a student network in the low-data target one. The knowledge transfer is enabled via an alignment module, which estimates entity correspondence across languages based on temporal patterns. Firstly, to alleviate the limited language alignments (\textbf{Challenge \#1}), such a knowledge distillation process is mutually paced over time. This means, on one hand, we encourage the mutually interactive learning between the teacher and student. Concretely, the alignment module between the teacher and the student learns to generate pseudo alignment between TKGs to maximally expand the upper bound of knowledge transfer. And subsequently, it empowers the student to encode more informative knowledge in target TKG, which can in turn boost the alignment module to explore more reasonable alignments as the bridge across TKGs. One the other hand, inspired by self-paced learning~\cite{spl-1,spl-2}, we make the generations as a progressively easy-to-hard process over time. We start from generating reliable pseudo data with high confidence. As time goes by, we then gradually increase the generation amount by relieving the restriction over time. Secondly, to inhibit the temporal knowledge mismatch (\textbf{Challenge \#2}), the attention module can estimate the graph alignment strength distribution over time. This is achieved by a temporal cross-lingual attention in terms of the local graph structure and temporal-evolving patterns of aligned entities. As such, it can dynamically control the negative effect and suppress noise  propagation from the source TKG. Moreover, we provide a theoretical convergence guarantee for the training objective on both initial ground-truth data and pseudo data. To evaluate \model, we conduct extensive experiments of 12 cross-lingual TKG transfer tasks in multilingual EventKG dataset~\cite{EventKG}. Our empirical results show that the \model method outperforms state-of-the-art baselines in both with and without alignment noise settings, where only $20\%$ of temporal events in the target KG and $10\%$ of cross-lingual alignments are preserved.


To sum up, our contributions are three-fold:

\begin{itemize}[leftmargin = 15pt]
    \item \textbf{Problem formulation}: We propose the cross-lingual temporal knowledge graph reasoning task, to boost the temporal reasoning performance in target TKG by transferring knowledge from source TKG;
    \item \textbf{Novel framework}: We propose a novel \model framework, which enables the mutually-paced learning between the teacher and student networks, to promote both pseudo alignments and knowledge transfer reliability. Besides, \model involves a dynamic alignment estimation across TKGs that inhibits the influence of temporal knowledge discrepancy.
    \item \textbf{Extensive evaluations}: Empirically, extensive experiments on 12 cross-lingual TKG transfer tasks in multilingual EventKG benchmark dataset demonstrate the effectiveness of \model.
\end{itemize}


\section{Preliminaries and Notations}
In this section, we formally define the cross-lingual temporal knowledge graph reasoning task, and summarize the notations in Table~\ref{tb:notation}. A temporal knowledge graph can be defined as follows:

\begin{table}[t]
    \centering
    \caption{Symbols and Notations.}
    \label{tb:notation}
    \small
    \resizebox{1.0\linewidth}{!}{
    \fontsize{8.5}{11}\selectfont
    \begin{tabular}{c|c}
    \toprule
    \textbf{Symbol} & \textbf{Definition} \\ \midrule
    $(e, r, e^\prime, t)$ & A quadruple in TKG. \\
    $\mathcal{G}_s$, $\mathcal{G}_t$  & Source TKG and Target TKG.\\
    $e_s$, $e_t$ & Entities in the source and target TKGs. \\
    $\Gamma_{s\leftrightarrow t}$ & Alignments between the source and target TKGs. \\
    $\Tilde{\mathcal{G}}_t$, $\Tilde{\Gamma}_{s\leftrightarrow t}$ & Incomplete target TKG and alignments. \\
    $\Tilde{\mathcal{G}}_t^{ST}$, $\Tilde{\Gamma}_{s\leftrightarrow t}^{ST}$ & Pseudo target TKG and  pseudo alignment. \\
    $f(\cdot; \Theta_s)$ & Teacher network on the source TKG.  \\
    $f(\cdot; \Theta_t)$ & Student network on the target TKG.  \\
    $g(e_s, e_t, t; \Phi)$ & Alignment module measuring correspondence of $(e_s, e_t)$ at $t$. \\
    $\mathcal{L}_{\Tilde{\mathcal{G}}_t}$, $\mathcal{L}_{\Tilde{\mathcal{G}}_t^{ST}}$ & Reasoning loss on groundtruth/pseudo target TKG. \\
    $\mathcal{L}_{\Tilde{\Gamma}_{s\leftrightarrow t}}$, $\mathcal{L}_{\Tilde{\Gamma}_{s\leftrightarrow t}^{ST}}$ & Alignment loss on groundtruth/pseudo alignment pairs. \\
    $\mathcal{L}_{s \rightarrow t}$ & Cross-lingual reasoning loss from source TKG to target TKG. \\
    $\mathcal{L}_{s \rightarrow t}^{ST}$ & Cross-lingual reasoning loss on both groundtruth and pseudo data. \\
    \bottomrule
    \end{tabular}}
\end{table}

\begin{definition}[{\bf Temporal Knowledge Graph}]  A temporal knowledge graph (TKG) is denoted as $\mathcal{G} = \{(e, r, e^\prime, t) | t \leq T\} \subseteq \mathcal{E} \times \mathcal{R} \times \mathcal{E} \times \mathcal{T}$, where $\mathcal{E}$ denotes the entities set, $\mathcal{R}$ denotes the relation set, $\mathcal{T}$ denotes the timestamp set, and $T$ denotes the latest update time. Each quadruple $(e, r, e^\prime, t)$ refers to an event that a subject entity $e \in \mathcal{E}$ has a relation $r \in \mathcal{R}$ with an object entity $e^\prime \in \mathcal{E}$ at timestamp $t \in \mathcal{T}$.
\end{definition}


\begin{definition}[{\bf Multilingual TKGs and Alignments}] To denote multilingual TKGs, we further utilize subscript to represent specific languages, i.e., $\mathcal{G}_s$ denotes TKG in the source language and $\mathcal{G}_t$ denotes TKG in the target language. The corresponding entities can be denoted as $e_s$ and $e_t$ respectively. Given two different languages $s$,$t$, we have the cross-lingual alignment set $\Gamma_{s\leftrightarrow t}$. To be more practical, we further assume the TKG in target language $\mathcal{G}_t$ and alignment set $\Gamma_{s\leftrightarrow t}$ are incomplete: $\Tilde{\mathcal{G}}_t$, $\Tilde{\Gamma}_{s\leftrightarrow t}$. 
\end{definition}


\begin{figure*}
    \centering
    \includegraphics[width = 0.9\linewidth]{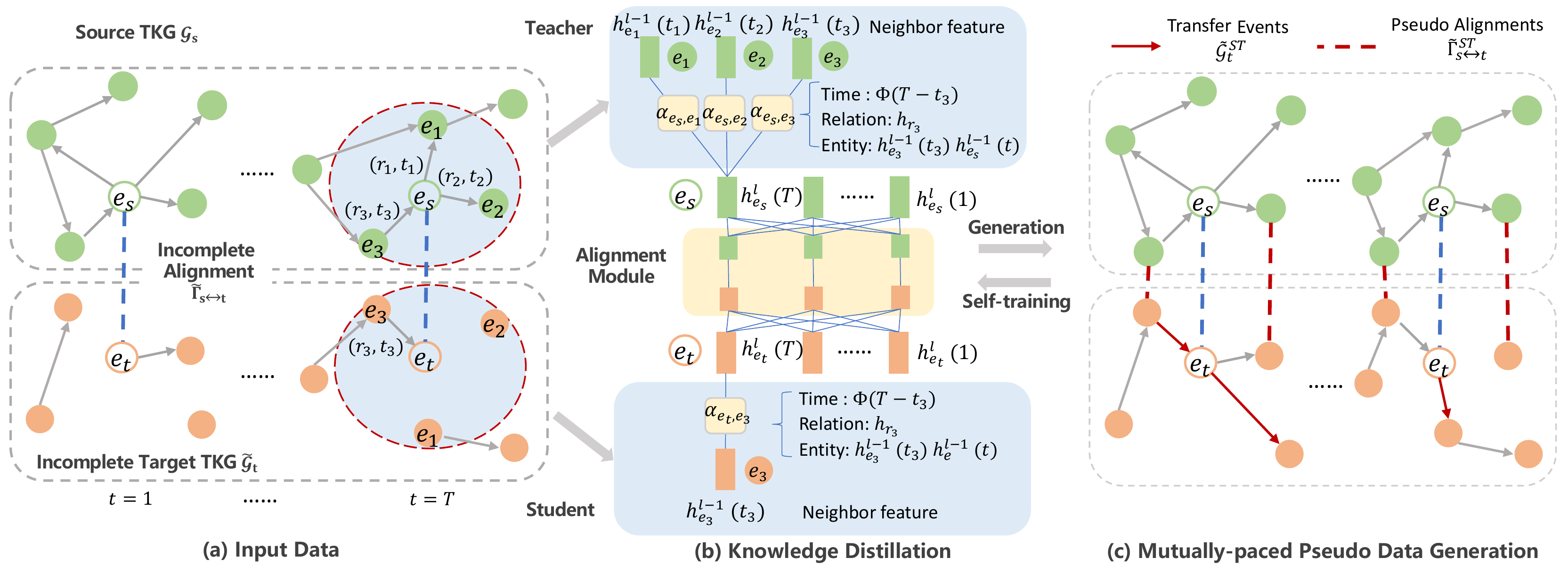}
    \vspace{-2mm}
    \caption{An overview of \model. (a) The source TKG is more complete than the target TKG, and the cross-lingual alignments are also scarce; (b) A teacher/student representation module to represent source/target TKG, and an alignment module for knowledge transfer; (c) Mutually-paced knowledge distillation between knowledge transfer and pseudo alignment generation.}
    \label{fig:framework}
    \vspace{-3mm}
\end{figure*}

Based on the definition above, we formalize our cross-lingual reasoning task on TKGs as follows:
\begin{definition} [{\bf Cross-lingual reasoning on TKGs}]  Given the TKG $\mathcal{G}_s$ in the source language and the incomplete TKG $\Tilde{\mathcal{G}}_t$ in the target language before the latest update time $T$, and the incomplete cross-lingual alignment $\Tilde{\Gamma}_{s\leftrightarrow t}$, we aim to predict future events in the target TKG after time $T$. Concretely, we aim to predict missing entity in each future quadruple: $\{(e_t,r,?,t)~\text{or}~(?,r,e_t^\prime,t) | t > T\}$ in the target TKG.
\end{definition}


\section{Methodology}
In this section, we present the proposed \model framework for the cross-lingual temporal knowledge graph reasoning task.

\subsection{Overview}
Figure~\ref{fig:framework} shows an overview of \model. Given the TKGs in source language and target language, the teacher network and the student network first represent the source and target TKGs in a temporally evolving uni-space respectively. To facilitate training of the student, the knowledge distillation is enabled by a cross-lingual alignment module and an explicit temporal event transfer process. To deal with the scarcity issue of cross-lingual alignments, we propose a pseudo alignment generation technique to facilitate the knowledge distillation process, which is mutually-paced along with model training.  To address the temporal knowledge discrepancy issue, the alignment module pulls the aligned entities close to each other based on the alignment strength which is dynamically adjusted.

Section~\ref{sec:encoder} and Section~\ref{sec:align} introduce our teacher/student network and the knowledge distillation respectively, followed by Section~\ref{sec:st} which details how we generate pseudo alignments. Finally, Section~\ref{sec:obj} specifies our learning objective on both groundtruth data and pseudo data, and summarizes the training of \model.

\subsection{The Teacher/Student Network}
\label{sec:encoder}
We train two identical temporal representation modules on source and target TKGs with different parameters. The representation module $f(\cdot;\Theta)$ parameterized by $\Theta$ is designed to measure the plausibility of each quadruple, which represents each entity $e$ into a low-dimensional latent space at each time: $\mathbf{h}_{e}(t) \in \mathbb{R}^d$. On a TKG $\mathcal{G}$, entities $e \in \mathcal{E}$ are evolving, as they interact with different entities over time. Such temporally interacted entities are defined as temporal neighbors. Therefore, we aim to model the temporal pattern of each entity $e$ by encoding the changes of temporal neighbors. 

Towards this goal, $f(\cdot;\Theta)$ first samples temporal neighbors $\mathcal{N}_{e}(t)$ from the TKG for each entity $e\in \mathcal{E}$. $\mathcal{N}_{e}(t)$ consists of a set of the most recently interacted entities at time $t$. Then $f(\cdot;\Theta)$ attentively aggregates information from the temporal neighbors. Specifically, given the temporal neighbor $\mathcal{N}_{e}(t)$, we represent the entity $e$ as $\mathbf{h}_{e}(t)$ at time $t$:
\begin{equation}
    \small
    \mathbf{h}^l_{e}(t) = \sigma\left(\sum_{(e_i, r_i, t_i) \in \mathcal{N}_{e}(t)} \alpha^l_{e,e_i} \left(\mathbf{h}_{e_i}^{l-1}(t_i) \mathbf{W}\right)\right),
\end{equation}
\noindent where $l$ denotes the layer number, $\sigma(\cdot)$ denotes the activation function {\em ReLU}, $\alpha^l_{e,e_i}$ denotes the attention weight of entity $e_i$ to the represented entity $e$, and $\mathbf{W}$ is the trainable transformation matrix. To aggregate from history, $\alpha^l_{e,e_i}$ is supposed to be aware of entity feature, time delay and topology feature induced by relations. Thus, we design the attention weight $\alpha^l_{e,e_i}$ as follows:
\begin{equation}
    \small
      \alpha^l_{e,e_i} = \frac{\exp(q^l_{e,e_i})}{\underset{(e_k, r_k, t_k)\in \mathcal{N}_{e}(t)}{\sum}\exp(q^l_{e,e_k})}, ~~ q^l_{e,e_k} = \mathbf{a}\left(\mathbf{h}^{l-1}_{e} \| \mathbf{h}^{l-1}_{e_k} \| \mathbf{h}_{r_k} \| \kappa(t - t_k)\right),  
\end{equation}
\noindent
where $q^l_{e,e_k}$ measures the pairwise importance by considering the entity embedding, relation embedding and time embedding, $\mathbf{a}\in\mathbb{R}^{4d}$ is the shared parameter in the attention mechanism. Following~\cite{TGAT} we adopt random Fourier features as time encoding $\kappa(\Delta t)$ to reflect the time difference.

To measure plausibility of each possible quadruple, we utilize TransE~\cite{TransE} as the score function $f(e, r, e^\prime, t; \Theta) = -\|\mathbf{h}^l_{e}(t) + \mathbf{h}_r - \mathbf{h}^l_{e^\prime}(t) \|^2$, where true quadruples should have higher scores. To optimize the parameter $\Theta$ on a TKG $\mathcal{G}$, we set the objective to rank the scores of true quadruples higher than all other false quadruples produced by negative sampling:
\begin{equation}
    \small
    \mathcal{L}_{\mathcal{G}} = \underset{(e, r, e^\prime, t) \in \mathcal{G}}{\mathbb{E}} \left[\max(0, \lambda_1 - f(e, r, e^\prime, t; \Theta) + f(e, r, e^{-}, t;\Theta))\right],
    \label{eq:kgloss}
\end{equation}
\noindent
where $(e, r, e^{-}, t)$ is negative samples with object $e^{\prime}$ replaced by $e^{-}$, $\lambda_1$ is the margin to distinguish positive and negative quadruples. 

\subsection{Knowledge Distillation}
\label{sec:align}
The incomplete target TKG, $\Tilde{\mathcal{G}}t$, can be used to train the corresponding parameter $\Theta_t$ through minimization of $\mathcal{L}{\Tilde{\mathcal{G}}_t}$. However, the low-resource nature of the target language often results in an incomplete target TKG, leading to suboptimal $\Theta_t$. In light of this, we propose a knowledge distillation approach to transfer temporal knowledge from the source TKG to the target TKG. The proposed approach consists of two components: an alignment module that enhances $\Theta_t$ using the more informative $\Theta_s$ learned from the source TKG, and an explicit temporal event transfer based on the improved parameters. This integrated approach aims to improve the completeness and quality of the target TKG by leveraging the knowledge contained in the source TKG.

\noindent \textbf{The Alignment Module}.
In general, the source parameters $\Theta_s$ provide a more informative representation of each entity $e \in \mathcal{E}$ compared to the target parameters $\Theta_t$. To take advantage of this, we utilize $\Theta_s$ to guide the optimization of $\Theta_t$ through the alignment module $g(\cdot;\Phi)$, which measures the correspondence between each pair of entities and is parameterized by $\Phi$.

Directly pulling embeddings of aligned entities at all time steps can transfer misleading knowledge due to the temporal knowledge discrepancy. Therefore, the alignment module first utilizes a temporal attention layer to integrate information of each entity from history in both source and target TKGs, i.e., $\mathbf{H}^s_{e}(t), \mathbf{H}^t_{e}(t) \in\mathbb{R}^d$, then it pulls such integration $\mathbf{H}^s_{e}(t)$ close to $\mathbf{H}^t_{e}(t)$ instead of the initial $\mathbf{h}^s_{e}(t)$ and $\mathbf{h}^t_{e}(t)$. Moreover, the temporal integration $\mathbf{H}^s_{e}(t)$ and $\mathbf{H}^t_{e}(t)$ also encode the temporal evolution information for each entity, which can be utilized to estimate the adaptive alignment strength at different time to improve the alignment module. Concretely, the temporal integration is learned by:
\begin{equation}
\begin{aligned}
    \small
    \mathbf{H}^s_{e}(t) &= \text{Temporal-Attn}(\mathbf{h}^s_{e}(1), \mathbf{h}^s_{e}(2), \cdots, \mathbf{h}^s_{e}(t)), \\
    \mathbf{H}^t_{e}(t) &= \text{Temporal-Attn}(\mathbf{h}^t_{e}(1), \mathbf{h}^t_{e}(2), \cdots, \mathbf{h}^t_{e}(t)),
\end{aligned}
\end{equation}
\begin{equation}
    \small
    g(e_s, e_t, t; \Phi) = \frac{\mathbf{H}^s_{e}(t) \cdot \mathbf{H}^t_{e}(t)}{\|\mathbf{H}^s_{e}(t)\|_2 \cdot \|\mathbf{H}^t_{e}(t)\|_2},
\end{equation}
\noindent
where $\text{Temporal-Attn}$ is the temporal attention network designed to integrate information on the temporal domain. The correspondence between each pair of entities $(e_s, e_t)$ across source and target languages at time $t$ is measured by $g(e_s, e_t, t; \Phi)$. As the temporal knowledge for aligned entities is not identical, the alignment strength between them should vary across time $t$. The alignment strength is strong when the two entities share similar information, and weak when the information is dissimilar or the alignment is unreliable. This variability is achieved through the design of a trainable weight $\beta_{e,t}$ to adjust the alignment strength for different entities at different times, which is generated by a cross-lingual attention layer:
\begin{equation}
    \small
    \beta_{e,t} = \text{Cross-Attn}(key = \mathbf{H}^t_{e}(1:T), query = \mathbf{H}^s_{e}(1:T))_{tt}.
\end{equation}

Due the page limitation, we refer readers to Appendix~\ref{ap:attn} for the detailed implementation of $\text{Temporal-Attn}(\cdot)$ and $\text{Cross-Attn}(\cdot)$.

To optimize the parameter $\Phi$ on the incomplete alignments $\Tilde{\Gamma}_{s \leftrightarrow t}$, we set the objective in order to rank the correspondence of true alignments higher than false alignments:
\begin{equation}
\small
    \mathcal{L}_{\Tilde{\Gamma}_{s \leftrightarrow t}} = \underset{\Tilde{\Gamma}_{s \leftrightarrow t}}{\mathbb{E}} \left[\underset{t \in \mathcal{T}}{\mathbb{E}} \left[ \beta_{e,t} \cdot \max(0, \lambda_2 - g(e_s, e_t, t; \Phi) + g(e_s, e_t^{-}, t; \Phi)) \right]\right],
    \label{eq:alignloss}
\end{equation}
\noindent
where the entity pair $(e_s, e_t) \in \Tilde{\Gamma}_{s \leftrightarrow t}$ is the aligned entities across languages, $(e_s, e_t^{-})$ is the negative samples, $\lambda_2$ is the margin value.

\noindent \textbf{Temporal Event Transfer}.
Cross-lingual alignments offer the potential to directly transfer temporal events towards the progressive completion of the target TKG. This is based on the premise that entities that are reliably aligned are likely to experience similar temporal events across languages, with the same relations. 

Given an aligned pair $(e_s, e_t)$, the temporal event $(e_t, r, e_t^?, t)$ or $(e_t^?, r, e_t, t)$ is added to the target TKG if the corresponding event $(e_s, r, e_s^?, t)$ or $(e_s^?, r, e_s, t)$ exists in the source TKG $\mathcal{G}_s$. To determine the missing entity $e_t^?$, we first verify if $(e_s^?, e_t^?)$ is present in the alignment set. If so, the temporal event is directly added to the target TKG. Otherwise, the updated student network $f(\cdot;\Theta_t)$ is utilized to predict the missing entity and the top-1 entity is utilized to complete the temporal event. We define the set of transferred temporal events in the target TKG as $\Tilde{\mathcal{G}}_t^{ST}$ for ease of discussion.

\subsection{Generating Pseudo Alignments}
\label{sec:st}
The limited amount of cross-lingual alignments negatively constrain the effect of the knowledge distillation process. In this section, we introduce how to generate pseudo alignments $\Tilde{\Gamma}_{s\leftrightarrow t}^{ST}$ with high confidence to boost cross-lingual transfer effectiveness.

To expand the range of alignments used in the knowledge transfer process, we generate pseudo-alignments with high confidence scores and incorporate them into the training data. The confidence score for each pair of entities $(e_s, e_t)$ is calculated as the average cosine similarity: $sim(e_s, e_t) = \underset{t}{\mathbb{E}}\left[g(e_s, e_t, t;\Phi)\right]$. While pair-wise similarity comparison is computationally intensive, we improve efficiency by first adding alignments for entities that are neighbors of already aligned entities $\Tilde{\mathcal{E}}t = \{e_t|(e_s, e_t) \in \Tilde{\Gamma}_{s\leftrightarrow t}\}$ in the target TKG, as they are likely to be represented well to produce reliable alignment. Following~\cite{bootEA}, we formulate the generation process as solving the following optimization problem:
\begin{equation}
    \begin{aligned}
    \max &\sum_{e_t \in \mathcal{N}(\Tilde{\mathcal{E}}_t)} \sum_{e_s\in\mathcal{E}_s} sim(e_s, e_t) \cdot \phi(e_s, e_t), \\
    s.t. & \sum_{e_s\in\mathcal{E}_s} \phi(e_s, e_t) = 1, ~~~ \sum_{e_t \in \mathcal{N}(\Tilde{\mathcal{E}}_t)} \phi(e_s, e_t) = 1,
    \end{aligned}
    \label{eq:pseudoalign}
\end{equation}
\noindent
where $\phi(e_s, e_t)$ is a binary indicator of whether to add $(e_s, e_t)$ as pseudo alignment, $\phi(e_s, e_t) = 1$ if we choose to add this pair, otherwise  $\phi(e_s, e_t) = 0$. The two constrains can guarantee each entity $e_t \in \mathcal{N}(\Tilde{\mathcal{E}}_t)$ is aligned to at most one entity in source language $e_s\in\mathcal{E}_s$. Finally, all pairs that satisfying $\phi(e_s, e_t) = 1$ can be viewed as candidates to be added into alignment data. We further select the top ones in terms of $sim(e_s, e_t)$ to control the total size of pseudo alignments. Notably, in each generation, the target entities to be aligned can already have the alignment, i.e., $\mathcal{N}(\Tilde{\mathcal{E}}_t) \bigcup \Tilde{\mathcal{E}}_t \neq \emptyset$. In this case, we can update the existing alignments with the pseudo ones to eliminate the possible alignment noise.

\label{sec:opt}
\begin{algorithm}[t]
\caption{The optimization process for \model.}
\label{al:training}
\small
\KwIn{Source TKG $\mathcal{G}_s$, incomplete target TKG $\Tilde{\mathcal{G}}_t$, incomplete alignment $\Tilde{\Gamma}_{s\leftrightarrow t}$.}
\KwOut{Student Model Parameter $\Theta_t$ for target TKG.}
Optimize $\Theta_s$ by minimizing $\mathcal{L}_{\mathcal{G}_s}$ on source TKG; \\
Initialize $\Theta_t$ $\leftarrow$ $\Theta_s$ for target TKG; \\
\While{model not converged}{
    \textbf{Optimize alignment module} $g(\cdot;\Phi)$:\\
    Minimize $\mathcal{L}_{s\rightarrow t}^{ST}$ in Eq.~\eqref{eq:stloss} w.r.t. alignment parameter $\Phi$; \\
    Transfer temporal events $\Tilde{\mathcal{G}}_t^{ST}$ based on updated $\Theta_t$;\\
    \textbf{Optimize student representation module} $f_t(\cdot;\Theta_t)$: \\
    \For{Each time step $T_i$ during training}{
        Prepare training data $\{(e_t,r,e_t^\prime, t) | (e_t,r,e_t^\prime, t) \in \Tilde{\mathcal{G}}_t \bigcup \Tilde{\mathcal{G}}_t^{ST}~\text{and}~T_i < t < T_{i+1}\}$\\
        Update $\Theta_t$ by minimizing $\mathcal{L}_{s\rightarrow t}^{ST}$ in Eq.~\eqref{eq:stloss};\\
    }
    Generate pseudo alignments $\Tilde{\Gamma}_{s\leftrightarrow t}^{ST}$ based on updated $\Phi$;
}
\end{algorithm}

\subsection{Mutually-paced Optimization}
\label{sec:obj}
\noindent \textbf{Learning Objective}.
Given a source TKG $\mathcal{G}_s$, the incomplete target TKG $\Tilde{\mathcal{G}}_t$, and the incomplete cross-lingual alignment $\Tilde{\Gamma}_{s\leftrightarrow t}$, the objective of cross-lingual temporal knowledge graph reasoning $\mathcal{L}_{s\rightarrow t}$ can be summarized as follows:
\begin{equation}
\mathcal{L}_{s\rightarrow t} = \mathcal{L}_{\Tilde{G}_t} + \mathcal{L}_{\Tilde{\Gamma}_{s\leftrightarrow t}},
\label{eq:gtloss}
\end{equation}
\noindent
where $\mathcal{L}_{\Tilde{G}_t}$ denotes knowledge graph reasoning loss which measures the correctness of each quadruple, $\mathcal{L}_{\Tilde{\Gamma}_{s\leftrightarrow t}}$ denotes the alignment loss which measures the distance of aligned entities in the uni-space. To enlarge the knowledge distillation effect, we progressively transfer temporal events $\Tilde{\mathcal{G}}_t$ and generate high-quality pseudo alignment $\Tilde{\Gamma}_{s\leftrightarrow t}^{ST}$. Therefore, the training objective on both ground-truth data and pseudo data $\mathcal{L}_{s\rightarrow t}^{ST}$ becomes:
\begin{equation}
    \small
    \begin{aligned}
        \mathcal{L}_{s\rightarrow t}^{ST} &= \frac{|\Tilde{\mathcal{G}}_t|}{|\Tilde{\mathcal{G}}_t| + |\Tilde{\mathcal{G}}_t^{ST}|} \cdot \mathcal{L}_{\Tilde{\mathcal{G}}_t} + \frac{|\Tilde{\mathcal{G}}_t^{ST}|}{|\Tilde{\mathcal{G}}_t| + |\Tilde{\mathcal{G}}_t^{ST}|} \cdot  \mathcal{L}_{\Tilde{\mathcal{G}}_t^{ST}} \\
        &+ \frac{|\Tilde{\Gamma}_{s\leftrightarrow t}|}{|\Tilde{\Gamma}_{s\leftrightarrow t}| + |\Tilde{\Gamma}_{s\leftrightarrow t}^{ST}|} \cdot \mathcal{L}_{\Tilde{\Gamma}_{s\leftrightarrow t}} + \frac{|\Tilde{\Gamma}_{s\leftrightarrow t}^{ST}|}{|\Tilde{\Gamma}_{s\leftrightarrow t}| + |\Tilde{\Gamma}_{s\leftrightarrow t}^{ST}|} \cdot \mathcal{L}_{\Tilde{\Gamma}_{s\leftrightarrow t}^{ST}},
    \end{aligned}
    \label{eq:stloss}
\end{equation}
\noindent
where $|\cdot|$ denotes the set size. Eq.~\eqref{eq:stloss} formulates the learning objective on both scarce data and pseudo data for cross-lingual temporal knowledge graph reasoning in target languages. We give the convergence analysis in the following theorem:

\begin{theorem}
Let $N$ denote the number of negative samples for optimization, $\epsilon$ denotes the portion of correct pseudo data, $\beta$ denotes the proportion of pseudo data to the initial ground-truth data. As the number of negative samples $N \rightarrow \infty$, the $\mathcal{L}_{s\rightarrow t}^{ST}$ converges to its limit with an absolute deviation decaying in $O(\frac{1+\epsilon}{1+\beta}\cdot N ^{-2/3})$.
\end{theorem}
\begin{proof}
Refer to Appendix~\ref{ap:proof}.
\end{proof}

\noindent \textbf{Mutually-paced Optimization and Generation}. 
The \model framework can be optimized by minimizing Eq.~\eqref{eq:stloss} w.r.t. $\Theta$ and $\Phi$ alternatively. The generated pseudo alignments can help the training of the representation modules by the knowledge distillation, and in turn transferring temporal events in the target TKG can improves alignment module by providing high-quality representations. In light of this, we propose a mutually-paced optimization and generation procedure. Generally speaking, we iteratively generate pseudo alignments and update the representation module and alignment module respectively. To be concrete, as shown in Algorithm~\ref{al:training}, we first update the alignment module $g(\cdot;\Phi)$ and transfer temporal events to transfer knowledge from source to target. Then we divide the time span into several time steps to update the student representation module from recent time step to far away ones. Finally, we generate the pseudo alignments, as the optimization $\Theta_t$ on all temporal events can improve the entity feature quality, which is beneficial for alignment prediction.  Algorithm~\ref{al:training} summarizes the training procedure. 


\section{Experiment}

\begin{table}[t]
\caption{Statistics of the datasets.}
\label{tb:data}
\small
\centering
\resizebox{1.0\linewidth}{!}{
    \fontsize{8.5}{11}\selectfont
\begin{tabular}{c|ccccc}
\toprule
\textbf{Languages} & \textbf{Entity} & \textbf{Relation} & \textbf{Quadruple} & \textbf{Train/Val/Test} & \textbf{Time} \\ \midrule
\textbf{English (EN)} & 34,416 & 105 & 602K & 602K/0K/0K & 28 \\
\textbf{French (FR)} & 32,546 & 105 & 580K & 580K/0K/0K & 28 \\ \midrule
\textbf{Spanish (ES)} & 31,808 & 105 & 316K & 114K/136K/66K & 40 \\
\textbf{German (DE)} & 27,657 & 105 & 268K & 97K/114K/56K & 40 \\
\textbf{Italian (IT)} & 23,734 & 94 & 236K & 84K/100K/51K & 40 \\
\textbf{Danish (DA)} & 15,710 & 94 & 125K & 48K/50K/26K & 40 \\
\textbf{Slovene (SL)} & 13,250 & 94 & 55K & 24K/21K/10K & 40 \\
\textbf{Bulgarian (BG)} & 3,508 & 105 & 23K & 8K/9K/6K & 40 \\ 
\bottomrule
\end{tabular}}
\end{table}
\begin{table*}[t]
\caption{Overall Performance without alignment noise. Average results on $5$ independent runs are reported. $*$ indicates the statistically significant improvements over the best baseline, with $p$-value smaller than $0.01$.}
\label{tb:clean}
\small
\centering
\resizebox{0.85\textwidth}{!}{
\fontsize{8.5}{11}\selectfont
\begin{tabular}{c|c|cccccccccccccc}
\toprule
\textbf{Models} & \textbf{Target} & \multicolumn{2}{c}{\textbf{ES}} & \multicolumn{2}{c}{\textbf{DE}} & \multicolumn{2}{c}{\textbf{IT}} & \multicolumn{2}{c}{\textbf{DA}} & \multicolumn{2}{c}{\textbf{BG}} & \multicolumn{2}{c}{\textbf{SL}} & \multicolumn{2}{c}{\textbf{Avg.}} \\ \hline
\textbf{} & \textbf{Source} & \textbf{MRR} & \textbf{H@10} & \textbf{MRR} & \textbf{H@10} & \textbf{MRR} & \textbf{H@10} & \textbf{MRR} & \textbf{H@10} & \textbf{MRR} & \textbf{H@10} & \textbf{MRR} & \textbf{H@10} & \textbf{MRR} & \textbf{H@10} \\ \hline
\rowcolor{LightCyan}
\textbf{RE-GCN w/o source} & \textbf{NA} & 14.31 & 31.85 & 16.32 & 34.19 & 14.59 & 31.64 & 14.19 & 31.24 & 10.27 & 23.44 & 9.33 & 21.63 & 13.17 & 29.00 \\ \hline
\multicolumn{16}{c}{ \textbf{Static KG embedding methods}} \\ \hline
\multirow{3}{*}{\textbf{TransE}~\cite{TransE}} & \textbf{EN} & 11.67 & 26.73 & 15.19 & 31.37 & 9.15 & 21.44 & 12.71 & 23.31 & 10.17 & 23.72 & 9.73 & 21.83 & 11.44 & 24.73 \\
 & \textbf{FR} & 12.37 & 27.79 & 14.01 & 28.30 & 11.38 & 23.19 & 10.05 & 22.10 & 11.88 & 23.01 & 10.63 & 22.44 & 11.72 & 24.47 \\
 \rowcolor{mygray}
 & \textbf{T.R.} & 0.84 & 0.86 & 0.89 & 0.87 & 0.70 & 0.71 & 0.80 & 0.73 & 1.07 & 1.00 & 1.09 & 1.02 & 0.88 & 0.85 \\ \hline
\multirow{3}{*}{\textbf{TransR}~\cite{TransR}} & \textbf{EN} & 11.88 & 28.66 & 16.01 & 32.01 & 8.14 & 22.07 & 13.34 & 24.73 & 10.33 & 23.51 & 8.89 & 22.12 & 11.43 & 25.52 \\
 & \textbf{FR} & 12.01 & 28.32 & 14.58 & 29.51 & 9.93 & 24.66 & 11.90 & 22.64 & 11.98 & 23.44 & 9.27 & 23.88 & 11.61 & 25.41 \\
 \rowcolor{mygray}
 & \textbf{T.R.} & 0.83 & 0.89 & 0.94 & 0.90 & 0.62 & 0.74 & 0.89 & 0.76 & 1.09 & 1.00 & 0.97 & 1.06 & 0.87 & 0.88 \\ \hline
\multirow{3}{*}{\textbf{DistMult}~\cite{DistMult}} & \textbf{EN} & 13.66 & 29.77 & 17.46 & 33.19 & 11.63 & 26.63 & 14.63 & 25.91 & 9.97 & 22.92 & 9.08 & 20.44 & 12.74 & 26.48 \\
 & \textbf{FR} & 12.58 & 28.73 & 16.03 & 31.81 & 12.12 & 27.76 & 11.64 & 22.97 & 9.01 & 23.77 & 10.13 & 21.07 & 11.92 & 26.02 \\
 \rowcolor{mygray}
 & \textbf{T.R.} & 0.92 & 0.92 & 1.03 & 0.95 & 0.81 & 0.86 & 0.93 & 0.78 & 0.92 & 1.00 & 1.03 & 0.96 & 0.94 & 0.91 \\ \hline
\multirow{3}{*}{\textbf{RotatE}~\cite{RotatE}} & \textbf{EN} & 12.99 & 28.89 & 19.87 & 35.46 & 15.62 & 30.14 & 13.44 & 25.79 & 11.10 & 22.98 & 11.37 & 23.99 & 14.07 & 27.88 \\
 & \textbf{FR} & 13.01 & 29.33 & 17.63 & 34.81 & 14.99 & 31.04 & 11.62 & 23.17 & 10.73 & 23.14 & 11.10 & 24.66 & 13.18 & 27.69 \\
 \rowcolor{mygray}
 & \textbf{T.R.} & 0.91 & 0.91 & 1.15 & 1.03 & 1.05 & 0.97 & 0.88 & 0.78 & 1.06 & 0.98 & 1.20 & 1.12 & 1.03 & 0.96 \\ \hline \multicolumn{16}{c}{ \textbf{Temporal KG embedding methods}} \\ \hline
\multirow{3}{*}{\textbf{TA-DistMult}~\cite{TA-DistMult}} & \textbf{EN} & 15.83 & 34.77 & 18.99 & 37.46 & 14.98 & 29.99 & 14.97 & 30.01 & 9.02 & 21.10 & 8.74 & 17.76 & 13.75 & 28.51 \\
 & \textbf{FR} & 16.61 & 35.83 & 17.81 & 37.96 & 15.58 & 31.21 & 13.21 & 28.58 & 9.63 & 22.91 & 9.03 & 18.83 & 13.65 & 29.22 \\
  \rowcolor{mygray}
 & \textbf{T.R.} & 1.13 & 1.11 & 1.13 & 1.10 & 1.05 & 0.97 & 0.99 & 0.94 & 0.91 & 0.94 & 0.95 & 0.85 & 1.04 & 1.00 \\ \hline
\multirow{3}{*}{\textbf{RE-Net}~\cite{Renet}} & \textbf{EN} & 17.58 & 37.97 & 19.03 & 39.46 & 15.88 & 33.69 & 15.03 & 34.77 & 12.01 & 25.72 & 11.07 & 25.64 & 15.10 & 32.88 \\
 & \textbf{FR} & 17.01 & 36.79 & 18.32 & 38.07 & 15.47 & 34.83 & 15.63 & 33.86 & 12.31 & 25.03 & 11.79 & 24.97 & 15.09 & 32.26 \\
  \rowcolor{mygray}
 & \textbf{T.R.} & 1.21 & 1.17 & 1.14 & 1.13 & 1.07 & 1.08 & 1.08 & 1.10 & 1.18 & 1.08 & 1.23 & 1.17 & 1.15 & 1.12 \\ \hline
\multirow{3}{*}{\textbf{RE-GCN}~\cite{RE-GCN}} & \textbf{EN} & 16.88 & 36.54 & 19.84 & 40.17 & 16.17 & 34.84 & 15.99 & 35.62 & 12.22 & 26.02 & 10.63 & 23.38 & 15.29 & 32.76 \\
 & \textbf{FR} & 17.14 & 37.01 & 19.63 & 41.01 & 16.44 & 35.61 & 15.03 & 33.19 & 11.91 & 25.13 & 11.09 & 22.77 & 15.21 & 32.45 \\ 
  \rowcolor{mygray}
 & \textbf{T.R.} & 1.19 & 1.15 & 1.21 & 1.19 & 1.12 & 1.11 & 1.09 & 1.10 & 1.17 & 1.09 & 1.16 & 1.07 & 1.16 & 1.12 \\ \hline 
 \multicolumn{16}{c}{ \textbf{Multilingual KG embedding methods}} \\ \hline
\multirow{3}{*}{\textbf{KEnS}~\cite{KEnS}} & \textbf{EN} & 15.98 & 33.91 & 17.33 & 37.62 & 14.41 & 31.44 & 14.47 & 29.61 & 12.88 & 26.77 & 11.03 & 24.99 & 14.35 & 30.72 \\
 & \textbf{FR} & 17.02 & 34.07 & 16.61 & 37.99 & 15.57 & 33.82 & 13.62 & 30.24 & 12.03 & 24.32 & 10.51 & 23.86 & 14.23 & 30.72 \\
  \rowcolor{mygray}
 & \textbf{T.R.} & 1.15 & 1.07 & 1.04 & 1.11 & 1.03 & 1.03 & 0.99 & 0.96 & 1.21 & 1.09 & 1.15 & 1.13 & 1.09 & 1.06 \\ \hline
\multirow{3}{*}{\textbf{AlignKGC}~\cite{AlignKGC}} & \textbf{EN} & 13.59 & 33.19 & 16.44 & 33.14 & 13.71 & 34.07 & 12.13 & 31.07 & 11.33 & 26.63 & 8.32 & 20.77 & 12.59 & 29.81 \\
 & \textbf{FR} & 13.90 & 34.71 & 17.14 & 34.81 & 14.97 & 33.65 & 12.07 & 30.44 & 10.92 & 25.31 & 9.64 & 21.28 & 13.11 & 30.03 \\
  \rowcolor{mygray}
 & \textbf{T.R.} & 0.96 & 1.07 & 1.03 & 0.99 & 0.98 & 1.07 & 0.85 & 0.98 & 1.08 & 1.11 & 0.96 & 0.97 & 0.98 & 1.03 \\ \hline
\multirow{3}{*}{\textbf{SS-AGA}~\cite{SS-AGA}} & \textbf{EN} & 15.11 & 32.19 & 16.49 & 36.14 & 14.83 & 33.31 & 12.27 & 30.68 & 12.99 & 27.03 & 11.55 & 25.07 & 13.87 & 30.74 \\
 & \textbf{FR} & 16.54 & 33.99 & 18.32 & 37.19 & 15.02 & 32.99 & 11.73 & 29.98 & 11.13 & 25.62 & 11.01 & 23.64 & 13.96 & 30.57 \\ 
  \rowcolor{mygray}
 & \textbf{T.R.} & 1.11 & 1.04 & 1.07 & 1.07 & 1.02 & 1.05 & 0.85 & 0.97 & 1.17 & 1.12 & 1.21 & 1.13 & 1.06 & 1.06 \\ \hline \midrule
\multirow{3}{*}{\textbf{\model*}} & \textbf{EN} & \textbf{19.51} & 41.55 & \textbf{22.84} & \textbf{49.30} & 17.18 & 37.62 & \textbf{18.79} & \textbf{40.01} & \textbf{14.33} & \textbf{30.13} & \textbf{13.87} & \textbf{30.30} & \textbf{17.75} & \textbf{38.15} \\
 & \textbf{FR} & 19.05 & \textbf{42.86} & 21.67 & 46.57 & \textbf{17.92} & \textbf{39.18} & 17.95 & 37.95 & 13.85 & 29.27 & 12.54 & 27.36 & 17.16 & 37.20 \\
  \rowcolor{mygray}
 & \textbf{T.R.} & 1.35 & 1.33 & 1.36 & 1.40 & 1.20 & 1.21 & 1.29 & 1.25 & 1.37 & 1.27 & 1.42 & 1.33 & 1.33 & 1.3 \\ \hline
\rowcolor{LightCyan}  \textbf{Gains} & & 11\% &	13\%&	15\%&	20\%&	9\%&	10\%&	18\%&	12\%&	10\%&	11\%&	18\%&	18\%&	16\%&	16\% \\
 \bottomrule
\end{tabular}}
\end{table*}

We evaluate \model on EventKG data~\cite{EventKG} including 2 source languages and 6 target languages, and we aim to answer the following research questions:
\begin{itemize}[leftmargin = 15pt]
    \item \textbf{RQ1}: How does \model perform compared with state-of-the-art models on the low-resource target languages?
    \item \textbf{RQ2}: How do reliability of alignment information (with various noise ratio) affect model performances?
    \item \textbf{RQ3}: How do each component and important parameters affect \model performance?
\end{itemize}

\subsection{Datasets} 
We evaluate \model by 12 cross-lingual TKG transfer tasks on EventKG data~\cite{EventKG}, which is a multilingual TKG including 2 source languages and 6 target languages. For each language, we collect events during 1980 to 2022 and split the time span into 40 time steps for training, validation and testing (28/4/8). Table~\ref{tb:data} shows the dataset statistics. We describe the dataset details in Appendix~\ref{ap:data}.

\subsection{Experimental Setup}

\noindent \textbf{Baselines}.
We compare ten state-of-the-art baselines from three related areas. We describe the baseline details in Appendix~\ref{ap:baseline}.
\begin{itemize}[leftmargin = 15pt]
    \item Static KG embedding methods: \textbf{TransE}~\cite{TransE}, \textbf{TransR}~\cite{TransR}, \textbf{DistMult}~\cite{DistMult}, \textbf{RotatE}~\cite{RotatE};
    \item Temporal KG embedding methods: \textbf{TA-DistMult}~\cite{TA-DistMult}, \textbf{RE-NET}~\cite{Renet}, \textbf{RE-GCN}~\cite{RE-GCN}; 
    \item Multilingual KG embedding methods: \textbf{KEnS}~\cite{KEnS}; \textbf{AlignKGC}~\cite{AlignKGC}; \textbf{SS-AGA}~\cite{SS-AGA}. 
\end{itemize}

\noindent \textbf{Evaluation Protocol and Metrics}.
For each prediction $(e, r, ?, t)$ or $(?, r, e, t)$, we rank missing entities to evaluate the performance. Following~\cite{RE-GCN}, we adopt raw mean reciprocal rank ({\em MRR}) and raw Hits at 10 ({\em H@10}) as evaluation metrics. To quantitatively compare how well the transferred knowledge from the source languages can improve predictions on the low-resource languages, we adopt Transfer Ratio ({\em T.R.}) to evaluate the average improvement of each 
method over the best baseline without knowledge transferring, i.e.:
\begin{equation}
    \small
    T.R. (t_i) = \frac{1}{|S|} \sum_{s_i \in \mathcal{S}} \frac{\text{Model}(s_i \rightarrow t_i)}{\text{BestBaseline}(t_i)}
\end{equation}
\noindent
where $t_i$ denotes each target language, $s_i \in \mathcal{S}$ denotes each source language, and $\text{BestBaseline}(t_i)$ denotes the best baseline performance on the target language $t_i$ without any knowledge transferring, i.e., {\em RE-GCN w/o source}.

\noindent \textbf{Implementation}.
To simulate scarce setting, by default, we utilize $10\%$ alignments and $20\%$ events of target TKG by random selection. For static/temporal KG embedding methods, we merge source graph and target graph by adding one new type of relation (alignment). For multilingual baselines, we train them on 1-to-1 knowledge transferring (instead of the original setting) for fair comparison. We introduce implementation details of baseline models and \model in Appendix~\ref{ap:implementation}. Code and data are open-source and available at \url{https://github.com/amzn/mpkd-thewebconf-2023}.

\begin{figure}[t]
    \centering
    \includegraphics[width = 1.0 \linewidth]{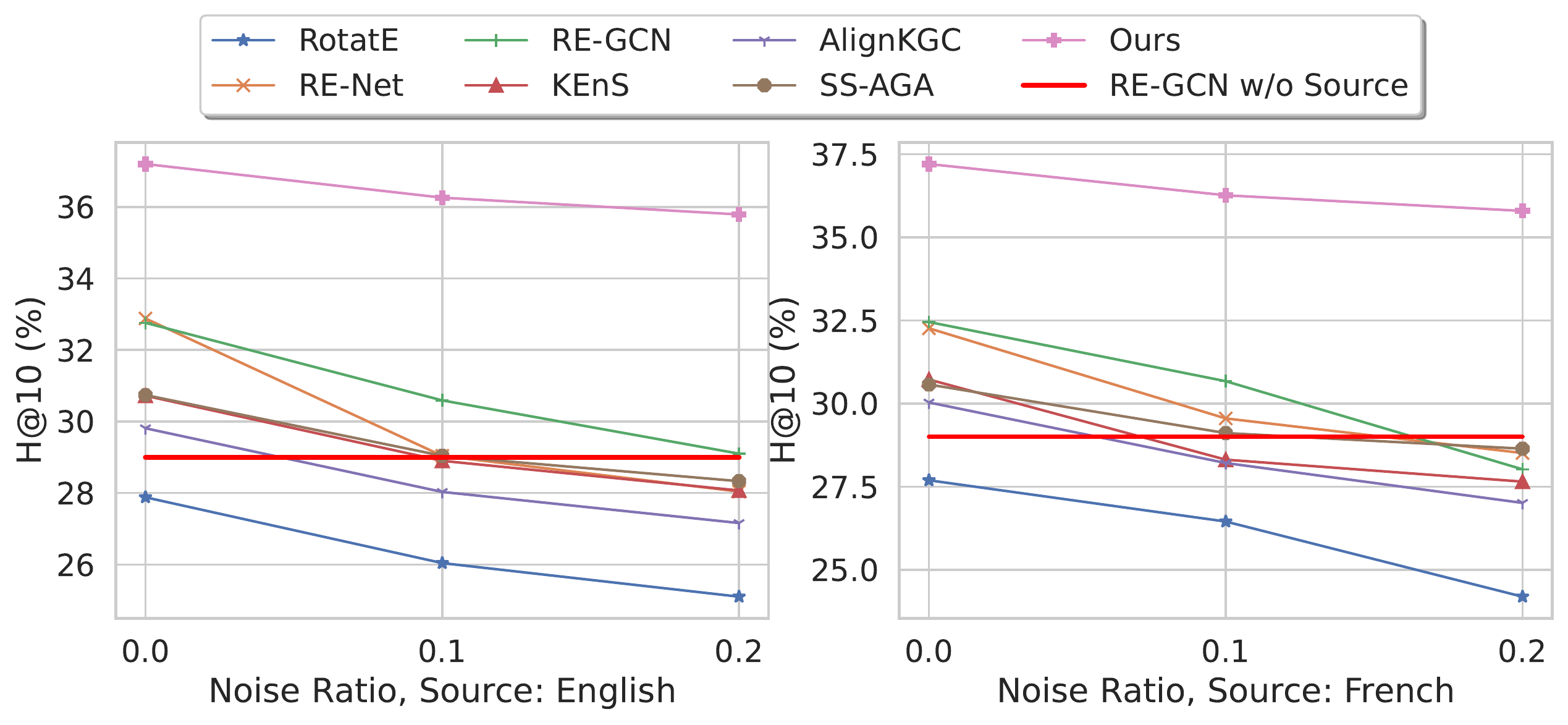}
    \caption{Experimental results under various alignment noise ratios. Average H@10 on 6 target languages are reported. \model achieves relatively robust results, with only $3.7\%$ relative drop, others have over $10\%$ drop.}
    \label{fig:noise}
    \vspace{-5mm}
\end{figure}

\subsection{Experiments on Cross-lingual Reasoning (RQ1)}
We first evaluate the model performance with incomplete cross-lingual alignments, where we randomly preserve $10\%$ alignments of the target entities for distilling knowledge. Table~\ref{tb:clean} reports the overall results for the cross-lingual experiments. By utilizing only $10\%$ cross-lingual alignments, \model achieves $33\%$ ({\em MRR}) and $30\%$ ({\em H@10}) relative improvement over best baseline without the knowledge transferring ({\em RE-GCN w/o source}) on average, demonstrating the effectiveness of \model in modeling alignments for knowledge transferring. Compared with ten baselines using alignments, \model still achieves  relative $14\%$ relative improvements over the second best results. Specifically, we have the following observations:
\begin{itemize}[leftmargin = 5pt]
\item Static baseline ({\em TransE, TransR, DistMult, RotatE}) fail to beat {\em RE-GCN w/o source}, although using alignments, due the insufficient modeling of temporal information. Similarly, multilingual methods ({\em KEnS, AlignKGC, SS-AGA}) also produce unsatisfying results; 

\item All temporal baselines ({\em TA-DistMult, RE-Net, RE-GCN}) manage to beat {\em RE-GCN w/o source}, as the modeling of both temporal evolution and cross-lingual alignment can facilitate the representation learning of target entities. But the improvements are marginal compared with our model, as the effect of knowledge distillation is constrained by the limited amount of cross-lingual alignments; 

\item Our model consistently achieves the best performance. Through $10\%$ alignments, \model can progressively transfer temporal knowledge and generate pseudo alignments with high confidence to boost the effect and range of the knowledge distillation;

\item We also notice the uneven improvements across languages, (e.g., ~$40\%$ improvements for German, ~$20\%$ for Italian). We hypothesize it is because of various language dependencies with source languages.
\end{itemize}


\subsection{Experiments under Alignment Noises (RQ2)}
In reality, cross-lingual alignments can be obtained by human labeling or rule-based inference modules, which may introduce indispensable noises. We evaluate how the reliability of alignment information affects baseline models and \model. In this experiment, we still utilize $10\%$ alignments. To simulate unreliable alignments, we select a subset of alignments (measured by {\em Noise Ratio}) and randomly change the aligned target entity to another entity without alignment information.

\begin{table}[t]
\caption{Ablation Studies.}
\label{tb:ablation}
\small
\centering
\resizebox{0.9\linewidth}{!}{
    \fontsize{8.5}{11}\selectfont
\begin{tabular}{c|c|cccccc}
\toprule
\textbf{Ablations} & \textbf{Target} & \multicolumn{2}{c}{\textbf{ES}} & \multicolumn{2}{c}{\textbf{SL}} & \multicolumn{2}{c}{\textbf{Avg.}} \\ \hline
\textbf{} & \textbf{Source} & \textbf{MRR} & \textbf{H@10} & \textbf{MRR} & \textbf{H@10} & \textbf{MRR} & \textbf{H@10} \\ \hline
\multirow{3}{*}{\textbf{\begin{tabular}[c]{@{}c@{}}\model w/o \\ Align. Strength \\ Control\end{tabular}}} & \textbf{EN} & 17.61 & 38.59 & 13.07 & 28.51 & 16.24 & 37.04 \\
\multicolumn{1}{l|}{} & \textbf{FR} & 17.07 & 39.46 & 13.03 & 28.14 & 16.33 & 36.61 \\
\multicolumn{1}{l|}{} &   \textbf{T.R.} & 1.21 & 1.23 & 1.27 & 1.21 & 1.21 & 1.22 \\ \hline
\multirow{3}{*}{\textbf{\begin{tabular}[c]{@{}c@{}}\model w \\ Pure Training\end{tabular}}} & \textbf{EN} & 17.09 & 37.03 & 11.89 & 26.25 & 14.81 & 31.90 \\
 & \textbf{FR} & 16.99 & 37.10 & 11.78 & 25.91 & 14.97 & 32.15 \\
 &   \textbf{T.R.} & 1.19 & 1.16 & 1.15 & 1.11 & 1.13 & 1.09 \\ \hline
\multirow{3}{*}{\textbf{\begin{tabular}[c]{@{}c@{}}\model w/o \\  Pseudo Align.\end{tabular}}} & \textbf{EN} & 17.97 & 38.04 & 12.31 & 27.83 & 15.98 & 34.83 \\
 & \textbf{FR} & 17.55 & 38.45 & 12.19 & 26.32 & 15.79 & 35.27 \\
 &  \textbf{T.R.} & 1.24 & 1.20 & 1.19 & 1.16 & 1.22 & 1.19 \\ \hline
\multirow{3}{*}{\textbf{\begin{tabular}[c]{@{}c@{}}\model w/o \\ Event Transfer\end{tabular}}} & \textbf{EN} & 18.79 & 39.03 & 13.07 & 28.07 & 16.03 & 36.74 \\
 & \textbf{FR} & 18.83 & 39.88 & 12.95 & 28.79 & 15.99 & 36.95 \\
 & \textbf{T.R.} & 1.31 & 1.24 & 1.27 & 1.21 & 1.27 & 1.24 \\ \hline
\multirow{3}{*}{\textbf{\model}} & \textbf{EN} & \textbf{19.51} & 41.55 & \textbf{14.33} & \textbf{30.13} & \textbf{17.75} & \textbf{38.15} \\
 & \textbf{FR} & 19.05 & \textbf{42.86} & 13.85 & 29.27 & 17.16 & 37.20 \\
 & \textbf{T.R.} & 1.35 & 1.33 & 1.37 & 1.27 & 1.33 & 1.3  \\
 \bottomrule
\end{tabular}}
\vspace{-5mm}
\end{table}
We vary the noise ratio from $0.0$ to $0.2$ to evaluate the models performance, as shown in Figure~\ref{fig:noise}. We report the average H@10 on 6 target languages by utilizing English TKG and French TKG respectively. As expected, with the increase of noise ratio, the performances of all compared models degrade, as the wrong alignment links mislead the knowledge transfer process. Most baselines fail to beat {\em RE-GCN w/o Source} even with $10\%$ noise, and all lose with $20\% noise$, which indicates that the quality of alignments significantly influences the model effectiveness in the cross-lingual TKG reasoning task. Notably, \model achieves relatively robust results, with only $3.7\%$ performance drop, while other strong baselines have over $10\%$ drop. This is because during the generation of pseudo alignments, \model can automatically replace those unreliable ones based on the confidence score. Also, in the alignment module, \model can assign small alignment strength to unreliable alignments.

\subsection{Model Analysis (RQ3)}
\noindent \textbf{Ablation Study}.
We evaluate performance improvements brought by the \model framework by following ablations: 
\begin{itemize}[leftmargin = 15pt]
\item {\bf \model w/o Align. Strength Control} uniformly set the alignment strength for all entities across all time steps; 
\item {\bf \model w Pure Training} optimizes the teacher-student framework without pseudo alignment generation and temporal event transfer; 
\item {\bf \model w/o Pseudo Align} eliminates the pseudo alignments generation process; 
\item {\bf \model w/o Event Transfer} eliminates the explicit transfer of temporal events.
\end{itemize}
Table~\ref{tb:ablation} reports the results measured by $H@10$. Each component leads to performance boost. \model with uniform alignment strength largely degrades performance, due to temporal knowledge discrepancy. \model without pseudo data generation achieves similar performance with temporal baselines {\em RE-Net, RE-GCN}, because of the limited amount of cross-lingual alignments. Bother generating pseudo alignments and explicitly transferring temporal events increase the performance, and combining them together in a mutually-paced procedure (in \model) can achieve the best results. 

\noindent \textbf{The Effect of Pseudo Alignments Ratio}.
To investigate the effects of the pseudo alignments on the reasoning performance, we vary the amount of pseudo alignments during training period and compare the corresponding performance measured by {\em H@10}, as shown in Figure~\ref{fig:pseudoratio}. The blue line and red line show the performances of single model on complete target TKG and single model on $20\%$ target TKG (our setting) respectively. From $0.1$, \model starts to generate and expand the initially available alignments. We observe a significant performance improvement, demonstrating the positive effects of the pseudo alignments. As expected, we find a performance decrease at $50\%$, as the added pseudo data with relatively low confidence start to introduce noise that hurt the performance.

\begin{figure}[t]
    \centering
    \begin{subfigure}[b]{0.495\linewidth}
    \centering
    \includegraphics[width = \linewidth]{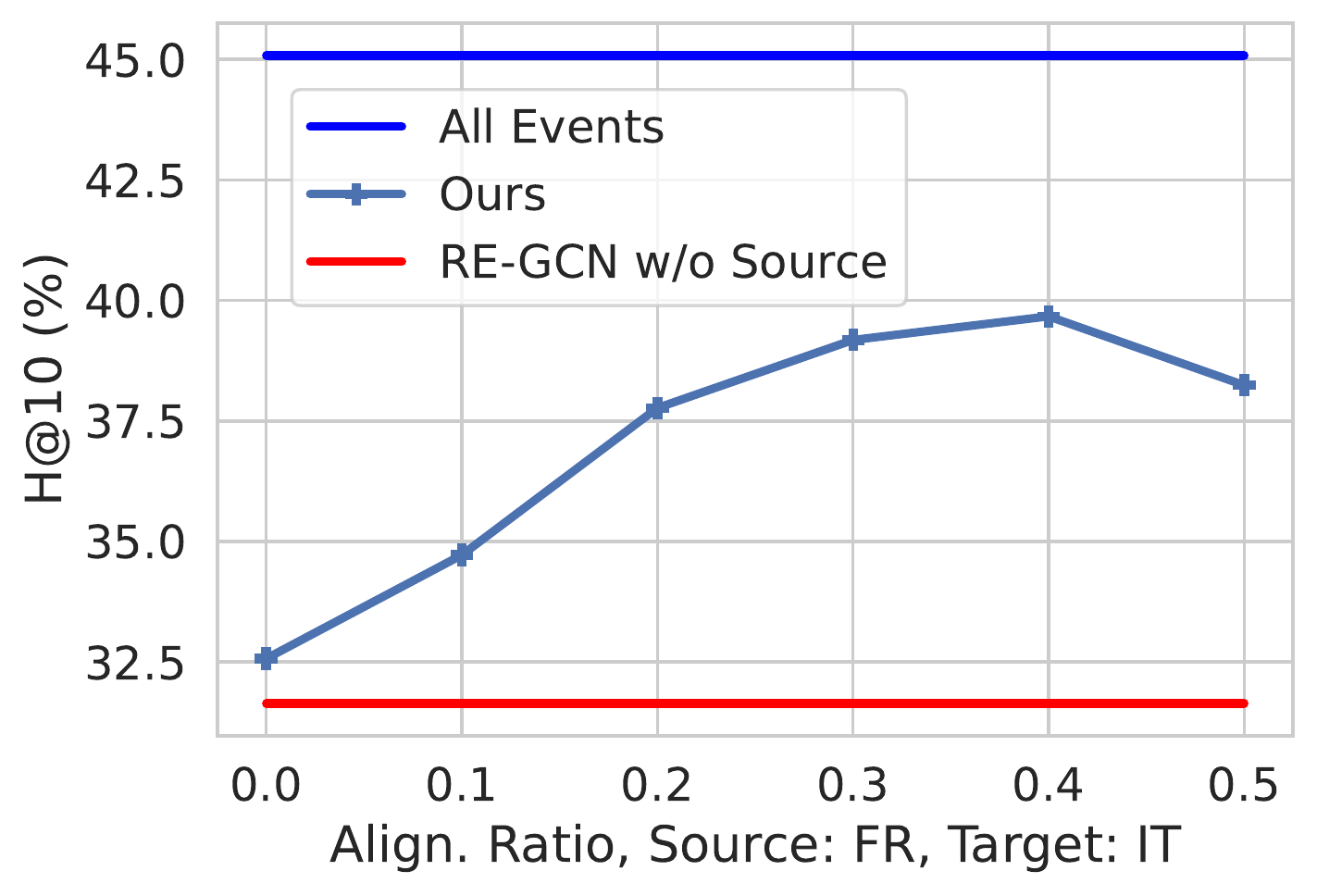}
    \caption{Pseudo align. analysis.}
    \label{fig:pseudoratio}
    \end{subfigure}
    \begin{subfigure}[b]{0.495\linewidth}
    \centering
    \includegraphics[width = \linewidth]{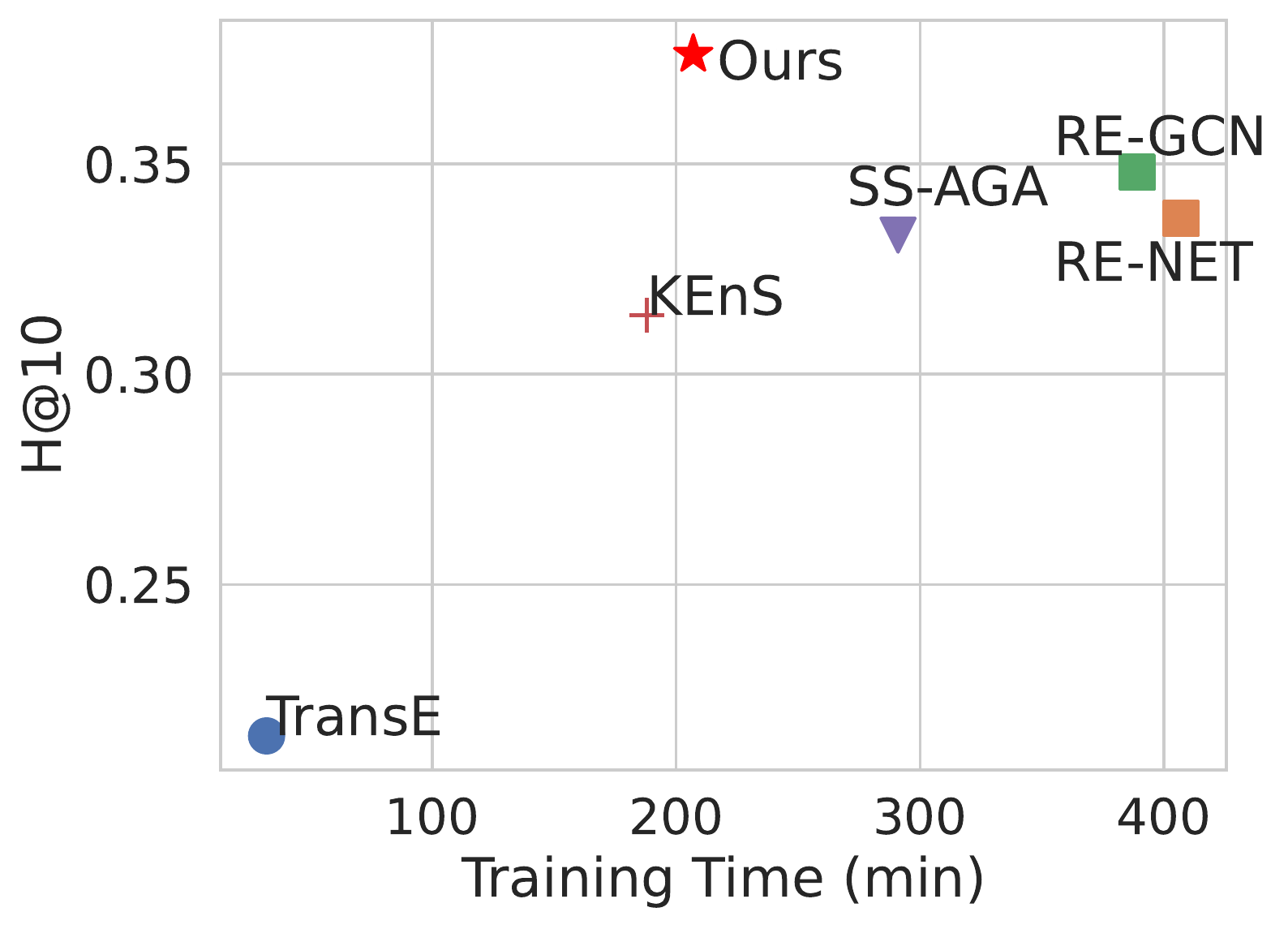}
    \caption{Efficiency analysis.}
    \label{fig:time}
    \end{subfigure}
    \caption{Efficiency analysis and pseudo alignment analysis.}
    \label{fig:analysis}
\end{figure}

\subsection{Efficiency Comparison}
To demonstrate the efficiency of \model framework, we train \model and baseline models from scratch on both target language and source language, and compare the training time. Figure~\ref{fig:time} shows that \model significantly outperforms baseline models with reasonable training time. More details are provided in Appendix~\ref{ap:time}.

\section{Related Work}

\noindent \textbf{Knowledge Graph Reasoning}.
Knowledge graph reasoning aims to predict missing facts to automatically complete KGs~\cite{DBPedia,WIKI,YAGO,acekg}. It is mostly formulated as measuring the correctness of factual samples and negative samples by specially designed score functions~\cite{TransE,RotatE,factKG,multiKG}. Recently, reasoning on temporal KGs attracts a lot of interests from the community~\cite{Know-Evolve,TA-DistMult,Renet,RE-GCN}. Compared with static KG reasoning task, the main challenge lies in how to incorporate time information. Several embedding-based methods have been proposed. They encode time-dependent information of entities and relations by decoupling embeddings into static component and time-varying component~\cite{TKGC_TE,GoelAAAI2020,metatkgr}, utilizing recurrent neural networks (RNNs) to adaptively learn the dynamic evolution from historical fact sequence~\cite{Renet,RE-GCN}, or learning a sequence of evolving representations from discrete knowledge graph snapshots~\cite{TKGC-ODE,DBKGE,Renet,RE-GCN}. However, all of the existing temporal KG reasoning models aim to extrapolate future facts based on relatively complete TKGs in high-resource languages, and how to boost reasoning performance for TKGs in low-resource languages through cross-lingual alignments is largely under-explored.

\noindent \textbf{Multilingual KG Reasoning}.
Entity alignment methods on KGs~\cite{EA1,DKGA,bright,selfKG,multi-network} can automatically enlarge the alignments by  predicting the correspondence between the two KGs. But most of them, if not all, require the relatively even completeness of two KGs to capture the structural similarities, which can not be satisfied in our case, as target TKGs are far from complete. Inspired by recent transfer learning advances~\cite{pan2009survey,li2018hierarchical,li2019exploiting,li2019sal,li2020learn,juan2022disentangle}, some recent works start to study the multilingual KG reasoning on static graphs~\cite{AlignKGC,KEnS,SS-AGA}, which aim to extract knowledge from several source KGs to boost the reasoning performance in the target KG, while they still require a sufficient amount of seed alignments and totally ignore the temporal information in our task.  We extend this line of works on TKGs, where transferring temporal knowledge is more complex.


\noindent \textbf{Self-training}.
Self-training is one of the learning strategies that addresses data scarcity issue by fully utilizing abundant unlabeled data~\cite{ST,ST1,li2021queaco}. Recent works start to study self-training strategy for graph data, as GNNs typically require large amount of data labeling~\cite{ST-GCN,CaGCN-st,Shift}. In summary, most efforts are put on node classification problem, where node labels are largely unavailable. We focus on utilizing self-training technique to deal with link scarcity (events + alignments), which is also a bottleneck for improving the performance on graphs.

\section{Conclusion}
In this paper, we studied a realistic but underexplored cross-lingual temporal knowledge graph reasoning problem, which aims at facilitating TKG reasoning in low-resource languages by distilling knowledge from a corresponding TKG in high-resource language through a small set of entity alignments as bridges. To this end, we proposed a novel mutually-paced teacher student framework, namely \model. During training, \model iteratively generates pseudo alignments to expand the cross-lingual connection, as well as transfers temporal facts to facilitate student model training in low-resource languages.  Our alignment module is learned to adjust the alignment strength for different entities at different time, thereby maximizing the benefits of knowledge transferring. We empirically validated the effectiveness of \model on 12 language pairs of EventKG data, on which the proposed framework significantly outperforms an extensive set of state-of-the-art baselines. 

\begin{acks}
The authors would like to thank the anonymous reviewers for their valuable comments and suggestions. Research reported in this paper was sponsored in part by DARPA award HR001121C0165, DARPA award HR00112290105, Basic Research Office award HQ00342110002, the Army Research Laboratory under Cooperative Agreement W911NF-17-20196, and Amazon.com Inc.
\end{acks}

\bibliographystyle{plain}
\bibliography{main}

\begin{thebibliography}{10}

\bibitem{TransE}
Antoine Bordes, Nicolas Usunier, Alberto Garcia-Duran, Jason Weston, and Oksana
  Yakhnenko.
\newblock Translating embeddings for modeling multi-relational data.
\newblock In {\em Advances in Neural Information Processing Systems}, 2013.

\bibitem{ICEWS18}
Elizabeth Boschee, Jennifer Lautenschlager, Sean O'Brien, Steve Shellman, James
  Starz, and Michael Ward.
\newblock Icews coded event data.
\newblock 2015.

\bibitem{EA1}
Muhao Chen, Yingtao Tian, Mohan Yang, and Carlo Zaniolo.
\newblock Multilingual knowledge graph embeddings for cross-lingual knowledge
  alignment.
\newblock In {\em IJCAI'17}.

\bibitem{KEnS}
Xuelu Chen, Muhao Chen, Changjun Fan, Ankith Uppunda, Yizhou Sun, and Carlo
  Zaniolo.
\newblock Multilingual knowledge graph completion via ensemble knowledge
  transfer.
\newblock In {\em EMNLP}, pages 3227--3238, 2020.

\bibitem{TGAT}
da~Xu, chuanwei ruan, evren korpeoglu, sushant kumar, and kannan achan.
\newblock Inductive representation learning on temporal graphs.
\newblock In {\em ICLR'20}, 2020.

\bibitem{TA-DistMult}
Alberto Garc{\'\i}a-Dur{\'a}n, Sebastijan Duman{\v{c}}i{\'c}, and Mathias
  Niepert.
\newblock Learning sequence encoders for temporal knowledge graph completion.
\newblock In {\em Proceedings of the 2018 Conference on Empirical Methods in
  Natural Language Processing}, 2018.

\bibitem{GoelAAAI2020}
Rishab Goel, Seyed~Mehran Kazemi, Marcus~A. Brubaker, and Pascal Poupart.
\newblock {Diachronic Embedding for Temporal Knowledge Graph Completion}.
\newblock In {\em AAAI'20}, 2020.

\bibitem{EventKG}
Simon Gottschalk and Elena Demidova.
\newblock Eventkg: the hub of event knowledge on the web and biographical
  timeline generation.
\newblock {\em Semantic Web}, 2019.

\bibitem{NCE}
M.~Gutmann and A.~Hyv\"arinen.
\newblock {N}oise-contrastive estimation: {A} new estimation principle for
  unnormalized statistical models.
\newblock In Y.W. Teh and M.~Titterington, editors, {\em Proc. Int. Conf. on
  Artificial Intelligence and Statistics (AISTATS)}, volume~9 of {\em JMLR W \&
  CP}, pages 297--304, 2010.

\bibitem{TKGC-ODE}
Zhen Han, Zifeng Ding, Yunpu Ma, Yujia Gu, and Volker Tresp.
\newblock Learning neural ordinary equations for forecasting future links on
  temporal knowledge graphs.
\newblock In {\em NeurIPS'21}, 2021.

\bibitem{ST1}
Junxian He, Jiatao Gu, Jiajun Shen, and Marc'Aurelio Ranzato.
\newblock Revisiting self-training for neural sequence generation.
\newblock {\em CoRR}, abs/1909.13788, 2019.

\bibitem{SS-AGA}
Zijie Huang, Zheng Li, Haoming Jiang, Tianyu Cao, Hanqing Lu, Bing Yin, Karthik
  Subbian, Yizhou Sun, and Wei Wang.
\newblock Multilingual knowledge graph completion with self-supervised adaptive
  graph alignment.
\newblock In {\em ACL'22}, 2022.

\bibitem{Renet}
Woojeong Jin, Meng Qu, Xisen Jin, and Xiang Ren.
\newblock Renet: Autoregressive structure inference over temporal knowledge
  graphs.
\newblock In {\em EMNLP}, 2020.

\bibitem{spl-1}
M.~Kumar, Benjamin Packer, and Daphne Koller.
\newblock Self-paced learning for latent variable models.
\newblock In J.~Lafferty, C.~Williams, J.~Shawe-Taylor, R.~Zemel, and
  A.~Culotta, editors, {\em NIPS'10}, volume~23. Curran Associates, Inc., 2010.

\bibitem{WIKI}
Julien Leblay and Melisachew~Wudage Chekol.
\newblock Deriving validity time in knowledge graph.
\newblock In {\em WWW'18}, 2018.

\bibitem{DBPedia}
Jens Lehmann, Robert Isele, Max Jakob, Anja Jentzsch, Dimitris Kontokostas,
  Pablo~N. Mendes, Sebastian Hellmann, Mohamed Morsey, Patrick van Kleef,
  S{\"o}ren Auer, and Christian Bizer.
\newblock {DBpedia} - a large-scale, multilingual knowledge base extracted from
  wikipedia.
\newblock {\em Semantic Web Journal}, 6(2):167--195, 2015.

\bibitem{belief}
Jinning Li, Huajie Shao, Dachun Sun, Ruijie Wang, Yuchen Yan, Jinyang Li,
  Shengzhong Liu, Hanghang Tong, and Tarek Abdelzaher.
\newblock Unsupervised belief representation learning with
  information-theoretic variational graph auto-encoders.
\newblock In {\em SIGIR'22}, 2022.

\bibitem{ST-GCN}
Qimai Li, Zhichao Han, and Xiao-Ming Wu.
\newblock Deeper insights into graph convolutional networks for semi-supervised
  learning.
\newblock In {\em AAAI'18}, 2018.

\bibitem{li2020learn}
Zheng Li, Mukul Kumar, William Headden, Bing Yin, Ying Wei, Yu~Zhang, and Qiang
  Yang.
\newblock Learn to cross-lingual transfer with meta graph learning across
  heterogeneous languages.
\newblock In {\em EMNLP'20}, pages 2290--2301, 2020.

\bibitem{li2019sal}
Zheng Li, Xin Li, Ying Wei, Lidong Bing, Yu~Zhang, and Qiang Yang.
\newblock Transferable end-to-end aspect-based sentiment analysis with
  selective adversarial learning.
\newblock In {\em EMNLP'19}, pages 4590--4600, Hong Kong, China, 2019.

\bibitem{li2018hierarchical}
Zheng Li, Ying Wei, Yu~Zhang, and Qiang Yang.
\newblock Hierarchical attention transfer network for cross-domain sentiment
  classification.
\newblock In {\em AAAI'18}, 2018.

\bibitem{li2019exploiting}
Zheng Li, Ying Wei, Yu~Zhang, Xiang Zhang, and Xin Li.
\newblock Exploiting coarse-to-fine task transfer for aspect-level sentiment
  classification.
\newblock In {\em AAAI'19}, 2019.

\bibitem{RE-GCN}
Zixuan Li, Xiaolong Jin, Wei Li, Saiping Guan, Jiafeng Guo, Huawei Shen,
  Yuanzhuo Wang, and Xueqi Cheng.
\newblock Temporal knowledge graph reasoning based on evolutional
  representation learning.
\newblock In {\em SIGIR}, 2021.

\bibitem{DBKGE}
Siyuan Liao, Shangsong Liang, Zaiqiao Meng, and Qiang Zhang.
\newblock Learning dynamic embeddings for temporal knowledge graphs.
\newblock In {\em WSDM '21}, 2021.

\bibitem{TransR}
Yankai Lin, Zhiyuan Liu, Maosong Sun, Yang Liu, and Xuan Zhu.
\newblock Learning entity and relation embeddings for knowledge graph
  completion.
\newblock In {\em AAAI'15}, 2015.

\bibitem{Shift}
Hongrui Liu, Binbin Hu, Xiao Wang, Chuan Shi, Zhiqiang Zhang, and Jun Zhou.
\newblock Confidence may cheat: Self-training on graph neural networks under
  distribution shift.
\newblock {\em CoRR}, abs/2201.11349, 2022.

\bibitem{factKG}
Lihui Liu, Boxin Du, Yi~Ren Fung, Heng Ji, Jiejun Xu, and Hanghang Tong.
\newblock Kompare: A knowledge graph comparative reasoning system.
\newblock In {\em KDD'21}, page 3308–3318, New York, NY, USA, 2021.
  Association for Computing Machinery.

\bibitem{multiKG}
Lihui Liu, Boxin Du, Jiejun Xu, Yinglong Xia, and Hanghang Tong.
\newblock Joint knowledge graph completion and question answering.
\newblock In {\em KDD '22}, page 1098–1108, New York, NY, USA, 2022.
  Association for Computing Machinery.

\bibitem{selfKG}
Xiao Liu, Haoyun Hong, Xinghao Wang, Zeyi Chen, Evgeny Kharlamov, Yuxiao Dong,
  and Jie Tang.
\newblock Selfkg: Self-supervised entity alignment in knowledge graphs.
\newblock {\em Proceedings of the ACM Web Conference 2022}, 2022.

\bibitem{YAGO}
Farzaneh Mahdisoltani, Joanna Biega, and Fabian~M. Suchanek.
\newblock Yago3: A knowledge base from multilingual wikipedias.
\newblock In {\em CIDR}, 2015.

\bibitem{pan2009survey}
Sinno~Jialin Pan and Qiang Yang.
\newblock A survey on transfer learning.
\newblock {\em IEEE Transactions on knowledge and data engineering},
  22(10):1345--1359, 2009.

\bibitem{KG4Social}
Jianwei Qian, Xiang-Yang Li, Chunhong Zhang, Linlin Chen, Taeho Jung, and Junze
  Han.
\newblock Social network de-anonymization and privacy inference with knowledge
  graph model.
\newblock {\em IEEE Transactions on Dependable and Secure Computing}, pages
  679--692, 2019.

\bibitem{KBQA}
Apoorv Saxena, Soumen Chakrabarti, and Partha Talukdar.
\newblock Question answering over temporal knowledge graphs.
\newblock In {\em ACL'21}, August 2021.

\bibitem{ST}
H.~J. Scudder.
\newblock Probability of error of some adaptive pattern-recognition machines.
\newblock {\em IEEE Trans. Inf. Theory}, 11:363--371, 1965.

\bibitem{misinfo}
Huajie Shao, Shuochao Yao, Andong Jing, Shengzhong Liu, Dongxin Liu, Tianshi
  Wang, Jinyang Li, Chaoqi Yang, Ruijie Wang, and Tarek Abdelzaher.
\newblock Misinformation detection and adversarial attack cost analysis in
  directional social networks.
\newblock In {\em ICCCN}, pages 1--11, 2020.

\bibitem{AlignKGC}
Harkanwar Singh, Soumen Chakrabarti, PRACHI JAIN, Sharod~Roy Choudhury, and
  Mausam.
\newblock Multilingual knowledge graph completion with joint relation and
  entity alignment.
\newblock In {\em Conference on Automated Knowledge Base Construction}, 2021.

\bibitem{bootEA}
Zequn Sun, Wei Hu, Qingheng Zhang, and Yuzhong Qu.
\newblock Bootstrapping entity alignment with knowledge graph embedding.
\newblock In {\em IJCAI'18}, page 4396–4402, 2018.

\bibitem{RotatE}
Zhiqing Sun, Zhi-Hong Deng, Jian-Yun Nie, and Jian Tang.
\newblock Rotate: Knowledge graph embedding by relational rotation in complex
  space.
\newblock In {\em International Conference on Learning Representations}, 2019.

\bibitem{Know-Evolve}
Rakshit Trivedi, Hanjun Dai, Yichen Wang, and Le~Song.
\newblock Know-evolve: Deep temporal reasoning for dynamic knowledge graphs.
\newblock In {\em ICML'17}.

\bibitem{andgan}
Haiwen Wang, Ruijie Wang, Chuan Wen, Shuhao Li, Yuting Jia, Weinan Zhang, and
  Xinbing Wang.
\newblock Author name disambiguation on heterogeneous information network with
  adversarial representation learning.
\newblock In {\em AAAI '20}, 2020.

\bibitem{RippleNet}
Hongwei Wang, Fuzheng Zhang, Jialin Wang, Miao Zhao, Wenjie Li, Xing Xie, and
  Minyi Guo.
\newblock Ripplenet: Propagating user preferences on the knowledge graph for
  recommender systems.
\newblock In {\em CIKM '18}, 2018.

\bibitem{DyDiff-VAE}
Ruijie Wang, Zijie Huang, Shengzhong Liu, Huajie Shao, Dongxin Liu, Jinyang Li,
  Tianshi Wang, Dachun Sun, Shuochao Yao, and Tarek Abdelzaher.
\newblock Dydiff-vae: A dynamic variational framework for information diffusion
  prediction.
\newblock In {\em Proceedings of the 44th International ACM SIGIR Conference on
  Research and Development in Information Retrieval}, SIGIR '21, 2021.

\bibitem{RETE}
Ruijie Wang, Zheng Li, Danqing Zhang, Qingyu Yin, Tong Zhao, Bing Yin, and
  Tarek Abdelzaher.
\newblock Rete: Retrieval-enhanced temporal event forecasting on unified query
  product evolutionary graph.
\newblock In {\em Proceedings of the ACM Web Conference 2022}, WWW '22, page
  462–472, 2022.

\bibitem{acekg}
Ruijie Wang, Yuchen Yan, Jialu Wang, Yuting Jia, Ye~Zhang, Weinan Zhang, and
  Xinbing Wang.
\newblock Acekg: A large-scale knowledge graph for academic data mining.
\newblock In {\em Proceedings of the 27th ACM International Conference on
  Information and Knowledge Management}, CIKM '18, 2018.

\bibitem{metatkgr}
Ruijie Wang, zheng li, Dachun Sun, Shengzhong Liu, Jinning Li, Bing Yin, and
  Tarek Abdelzaher.
\newblock Learning to sample and aggregate: Few-shot reasoning over temporal
  knowledge graphs.
\newblock In {\em Advances in Neural Information Processing Systems}, 2022.

\bibitem{proof}
Tongzhou Wang and Phillip Isola.
\newblock Understanding contrastive representation learning through alignment
  and uniformity on the hypersphere.
\newblock {\em CoRR}, abs/2005.10242, 2020.

\bibitem{CaGCN-st}
Xiao Wang, Hongrui Liu, Chuan Shi, and Cheng Yang.
\newblock Be confident! towards trustworthy graph neural networks via
  confidence calibration.
\newblock In {\em NeurIPS'21}.

\bibitem{TKGC_TE}
Chenjin Xu, Mojtaba Nayyeri, Fouad Alkhoury, Hamed Yazdi, and Jens Lehmann.
\newblock Temporal knowledge graph completion based on time series gaussian
  embedding.
\newblock In {\em The Semantic Web – ISWC 2020: 19th International Semantic
  Web Conference, Athens, Greece, November 2–6, 2020, Proceedings, Part I},
  page 654–671, 2020.

\bibitem{spl-2}
Miao Xu, Bingcong Li, Gang Niu, Bo~Han, and Masashi Sugiyama.
\newblock Revisiting sample selection approach to positive-unlabeled learning:
  Turning unlabeled data into positive rather than negative.
\newblock {\em CoRR}, abs/1901.10155, 2019.

\bibitem{DKGA}
Yuchen Yan, Lihui Liu, Yikun Ban, Baoyu Jing, and Hanghang Tong.
\newblock Dynamic knowledge graph alignment.
\newblock In {\em AAAI}, volume~35, 2021.

\bibitem{bright}
Yuchen Yan, Si~Zhang, and Hanghang Tong.
\newblock Bright: A bridging algorithm for network alignment.
\newblock In {\em Proceedings of the Web Conference 2021}, 2021.

\bibitem{multi-network}
Yuchen Yan, Qinghai Zhou, Jinning Li, Tarek Abdelzaher, and Hanghang Tong.
\newblock Dissecting cross-layer dependency inference on multi-layered
  inter-dependent networks.
\newblock In {\em CIKM'22}, 2022.

\bibitem{DistMult}
Bishan Yang, Wen-tau Yih, Xiaodong He, Jianfeng Gao, and li~Deng.
\newblock Embedding entities and relations for learning and inference in
  knowledge bases.
\newblock In {\em ICLR'14}.

\bibitem{polarization}
Chaoqi Yang, Jinyang Li, Ruijie Wang, Shuochao Yao, Huajie Shao, Dongxin Liu,
  Shengzhong Liu, Tianshi Wang, and Tarek Abdelzaher.
\newblock Hierarchical overlapping belief estimation by structured matrix
  factorization.
\newblock In {\em ASONAM'20}, 2020.

\bibitem{juan2022disentangle}
Juan Zha, Zheng Li, Ying Wei, and Yu~Zhang.
\newblock Disentangling task relations for few-shot text classification via
  self-supervised hierarchical task clustering.
\newblock In {\em EMNLP'22}, 2022.

\bibitem{li2021queaco}
Danqing Zhang, Zheng Li, Tianyu Cao, Chen Luo, Tony Wu, Hanqing Lu, Yiwei Song,
  Bing Yin, Tuo Zhao, and Qiang Yang.
\newblock Queaco: Borrowing treasures from weakly-labeled behavior data for
  query attribute value extraction.
\newblock In {\em CIKM'21}.

\bibitem{TKG4Rec2}
Yuyue Zhao, Xiang Wang, Jiawei Chen, Yashen Wang, Wei Tang, Xiangnan He, and
  Haiyong Xie.
\newblock Time-aware path reasoning on knowledge graph for recommendation.
\newblock {\em ACM Trans. Inf. Syst.}, 2022.

\bibitem{TKG4Rec}
Kun Zhou, Wayne~Xin Zhao, Shuqing Bian, Yuanhang Zhou, Ji-Rong Wen, and
  Jingsong Yu.
\newblock Improving conversational recommender systems via knowledge graph
  based semantic fusion.
\newblock In {\em KDD'20}, page 1006–1014, 2020.

\end{thebibliography}
\newpage
\appendix
\balance
\section{Appendix}

\subsection{Model Description}
\label{ap:attn}
In this section, we introduce the temporal attention layer and cross-lingual attention layer for entity alignments utilized in Section~\ref{sec:align}. We first introduce the general attention mechanism we utilized, then specify the two layers respectively.

Given two representation sequence from temporal domain: key sequence $\boldsymbol{H}_K = \{\boldsymbol{h}_K^{1}, \boldsymbol{h}_K^{2}, \cdots, \boldsymbol{h}_K^{T}\}$ and query sequence $\boldsymbol{H}_Q = \{\boldsymbol{h}_Q^{1}, \boldsymbol{h}_Q^{2}, \cdots, \boldsymbol{h}_Q^{T}\}$  from all time steps, we propose the following attention to calculate the pairwise importance:
\begin{equation}
    \small
    \mathbf{\beta} = \text{Attn}(key = \boldsymbol{H}_K, query = \boldsymbol{H}_Q) = \operatorname{softmax}\left(\frac{\boldsymbol{H}_Q \boldsymbol{W}^Q(\boldsymbol{H}_K \boldsymbol{W}^K)^T}{\sqrt{d}} + \boldsymbol{M}\right),
\end{equation}
\noindent
where $\boldsymbol{W}^Q$, $\boldsymbol{W}^L$ are trainable temporal parameters, $\mathbf{\beta}$ is learned temporal weight indicating pairwise importance, $d$ denotes dimension of input representations, and $\boldsymbol{M}$ is added to ensure auto-regressive setting, i.e., preventing future information affecting current state. We define $\boldsymbol{M}_{ij}=0$ if $i\leq j$, otherwise $-\infty$.

For {\em temporal attention} layer, we use $\boldsymbol{h}_e = \{\boldsymbol{h}_e(1), \boldsymbol{h}_e(2), \cdots, \boldsymbol{h}_e(T)\}$ for both query and key sequence to obtain the temporal attention weights $\beta$:
\begin{equation}
    \small
    \mathbf{\beta} = \text{Attn}(key = \boldsymbol{h}_e, query = \boldsymbol{h}_e),
\end{equation}
\noindent
then the desired $\boldsymbol{H}_e(t)$ is leaned as the combination of input sequence, where $\boldsymbol{W}^V$ is a trainable matrix.:
\begin{equation}
    \begin{aligned}
    \small
    \boldsymbol{H}_e(t) &= \text{Temporal-Attn}(\boldsymbol{h}_e(1), \cdots, \boldsymbol{h}_e(t)) \\
    &= \sum_{i = 1}^{t} \beta_{it} \boldsymbol{h}_e(i)\boldsymbol{W}^V
    \end{aligned}
\end{equation}

For {\em cross-lingual attention} layer, we use $\boldsymbol{H}^s_e = \{\boldsymbol{H}^s_e(1), \cdots, \boldsymbol{H}^s_e(T)\}$ in source language as query sequence and $\boldsymbol{H}^t_e = \{\boldsymbol{H}^t_e(1), \cdots, \boldsymbol{H}^t_e(T)\}$ in target language as key sequence to obtain the attention weights $\beta$:
\begin{equation}
    \small
    \mathbf{\beta}_{e,t} = \text{Attn}(key = \boldsymbol{H}^t_e, query = \boldsymbol{H}^s_e)_{tt},
\end{equation}
\noindent
where $\mathbf{\beta}_{e,t}$ is trainable weight to adjust the alignment strength of different entities at different time.

\subsection{Theorem Proof}
\label{ap:proof}
\begin{theorem}
Let $N$ denote the number of negative samples for optimization, $\epsilon$ denotes the ratio of correct pseudo data, $\beta$ denotes the ratio of pseudo data amount to the initial groundtruth data amount. As the number of negative samples $N \rightarrow \infty$, the $\mathcal{L}_{s\rightarrow t}^{ST}$ converges to its limit with an absolute deviation decaying in $O(\frac{1+\epsilon}{1+\beta}\cdot N ^{-2/3})$.
\end{theorem}
\begin{proof}
In representation learning, the margin loss has been widely adopted as the similarity metric. Without loss of generality, they can be expressed in the form of Noise Contrastive Estimation (NCE)~\cite{NCE}. Therefore, we express $\mathcal{L}_{\mathcal{G}}$ and $\mathcal{L}_{\Gamma_{s\leftrightarrow t}}$ in the form of Noise Contrastive Estimation (NCE) by introducing the negative sampling:
\begin{equation}
\scriptsize
    \mathcal{L}_{\mathcal{G}} \triangleq \mathbb{E}\left[-\log \frac{e^{f(\cdot;\Theta) / \tau}}{e^{f(\cdot;\Theta)/ \tau}+\sum_- e^{\left(f^-(\cdot;\Theta)\right) / \tau}}\right],
\end{equation}

\begin{equation}
\scriptsize
    \mathcal{L}_{\Gamma_{s\leftrightarrow t}} \triangleq \mathbb{E} \left[-\log \frac{e^{g(\cdot;\Phi)) / \tau}}{e^{g(\cdot;\Phi) / \tau}+\sum_- e^{g^-(\cdot;\Phi) / \tau}}\right],
\end{equation}
\noindent
for simplicity, $f^-(\cdot;\Theta)$ denotes the score for negative quadruple, and  $g^-(\cdot;\Phi)$ denotes score for negative alignment pair.

For our training objective $\mathcal{L}_{s \rightarrow t}^{ST}$, we show the convergence analysis of four terms one by one, then prove the overall convergence results. First of all, following~\cite{proof,selfKG}, let $N$ denote the number of negative samples per each quadruple, and we have:
\begin{equation}
\scriptsize
\begin{aligned}
&\lim _{N \rightarrow \infty}\left[\mathcal{L}_{\Tilde{\mathcal{G}}_t}-\log M\right]\\
&=-\frac{1}{\tau} \underset{(e_t, r, e_t^\prime, t) \in \Tilde{\mathcal{G}}_t}{\mathbb{E}}\left[f(\cdot;\Theta)\right] \\
&+\lim _{N \rightarrow \infty} \underset{e^-_t\in \mathcal{E}_t}{\underset{(e_t, r, e_t^\prime, t) \in \Tilde{\mathcal{G}}_t}{\mathbb{E}}}\left[\log \left(\frac{\lambda}{N} e^{f(\cdot;\Theta)/ \tau}+\frac{1}{N} \sum_- e^{f^-(\cdot;\Theta) / \tau}\right)\right]\\ 
&= -\frac{1}{\tau} \underset{(e_t, r, e_t^\prime, t) \in \Tilde{\mathcal{G}}_t}{\mathbb{E}}\left[f(\cdot;\Theta)\right] + \underset{(e_t, r, e_t^\prime, t) \in \Tilde{\mathcal{G}}_t}{\mathbb{E}} \left[\log \underset{e^-_t\in \mathcal{E}_t}{\mathbb{E}}\left[e^{f^-(\cdot;\Theta)}\right]\right],
\end{aligned}
\end{equation}
\noindent
where $\lambda$ denotes the duplicate quadruples co-existing in both incomplete $\Tilde{\mathcal{G}}_t$ and negative samples. The convergence speed is derived as follows:

For one side:
\begin{equation}
\scriptsize
    \mathcal{L}_{\Tilde{\mathcal{G}}_t}-\log N - \lim _{N \rightarrow \infty}\left[\mathcal{L}_{\Tilde{\mathcal{G}}_t}-\log N\right] \leq \frac{\lambda}{N}e^{\frac{2}{\tau}}.
\end{equation}

For another side:
\begin{equation}
    \scriptsize
    \lim _{N \rightarrow \infty}\left[\mathcal{L}_{\Tilde{\mathcal{G}}_t}-\log N\right] - \left[\mathcal{L}_{\Tilde{\mathcal{G}}_t}-\log N\right] \leq \frac{\lambda}{N}e^{2/\tau} + \frac{5}{4} N^{-\frac{2}{3}} e^{\frac{1}{\tau}}(e^{\frac{1}{\tau}}-e^{-\frac{1}{\tau}}).
\end{equation}

We then generalize the above results to the loss term on pseudo data. Suppose $\epsilon$ of the pseudo data are correct. Then we can have the following two inequality. For one side:
\begin{equation}
\scriptsize
\begin{aligned}
    &\mathcal{L}_{\Tilde{\mathcal{G}}_t^{ST}}-\log N - \lim _{N \rightarrow \infty}\left[\mathcal{L}_{\Tilde{\mathcal{G}}_t^{ST}}-\log N\right] \\
    &\leq \epsilon \underset{e\in\mathcal{E}_t}{\mathbb{E}}\left[\log \frac{\underset{e^{-}\in\mathcal{E}_t}{\mathbb{E}}\left[\frac{\lambda}{N}e^{1/\tau} + e^{f^-(\cdot;\Theta_t)/\tau}\right]}{\underset{e^{-}\in\mathcal{E}_t}{\mathbb{E}}e^{f^-(\cdot;\Theta_t)/\tau}} \right] \leq \epsilon \frac{\lambda}{N}e^{\frac{2}{\tau}}
\end{aligned}
\end{equation}

Therefore, for $\mathcal{L}_{s \rightarrow t}^{ST}$ in this side, we have:

\begin{equation}
    \scriptsize
    \mathcal{L}_{s \rightarrow t}^{ST}-\log N - \lim _{N \rightarrow \infty}\left[\mathcal{L}_{s \rightarrow t}^{ST}-\log N\right] \leq \frac{1+\epsilon}{1+\beta} \frac{\lambda}{N}e^{\frac{2}{\tau}},
\end{equation}
where $\beta$ is the ratio of pseudo data amount to groundtruth data amount during training.

Similarly, for another side, we have:
\begin{equation}
\scriptsize
\begin{aligned}
    \lim _{N \rightarrow \infty}\left[\mathcal{L}_{s \rightarrow t}^{ST}-\log N\right] &- \left[\mathcal{L}_{s \rightarrow t}^{ST}-\log N\right] \leq \frac{1+\epsilon}{1+\beta} \frac{\lambda}{N}e^{2/\tau} \\
    &+ \frac{1+\epsilon}{1+\beta} \frac{5}{4} N^{-\frac{2}{3}} e^{\frac{1}{\tau}}(e^{\frac{1}{\tau}}-e^{-\frac{1}{\tau}}).
\end{aligned}
\end{equation}

Therefore, we conclude that the $\mathcal{L}_{s \rightarrow t}^{ST}$ converges to its limit with an absolute deviation decaying in $O(\frac{1+\epsilon}{1+\beta}\cdot N ^{-2/3})$

\end{proof}

\subsection{Datasets}
\label{ap:data}
\noindent \textbf{Dataset Information}.
The commonly utilized benchmark TKGs are divided into two categories: temporal event graphs~\cite{ICEWS18} and knowledge graphs where temporally associated facts have valid periods~\cite{WIKI,YAGO,DBPedia}. In this paper, we mainly evaluate \model on the EventKG~\cite{EventKG}, which is a multilingual resource incorporating event-centric information extracted from several large-scale knowledge graphs such as Wikidata~\cite{WIKI}, DBpedia~\cite{DBPedia} and YAGO~\cite{YAGO}. Each temporal event is organized as $(e, r, e^\prime, t_s, t_e)$, where each piece of data is attached with a valid time period from start time $t_s$ to end time $t_e$. Following~\cite{Renet}, we preprocess the format such that each fact is converted to a sequence $\{(e, r, e^\prime, t_s), (e, r, e^\prime, t_s+1), \cdots, (e, r, e^\prime, t_e)\}$ from $t_s$ to $t_e$, with the minimum time unit as one step.

\noindent \textbf{Splitting Scheme}.
We collect events during 1980 to 2022, and noisy events of early years are removed. To construct multilingual TKGs, we first preserve important entities and relations by excluding infrequent ones that have less than $20$ events in each language. Then we collect the events and cross-lingual alignments.To guarantee the relation match, we only preserve relations appearing in English TKG. We split the time span into 40 equal time steps for training, validation and testing (28/4/8), where each time step roughly lasts for one year. To focus on the prediction on existing entities during training period, and eliminate the negative effects possibly caused by the randomly appearning new entities in val/test period, we only preserve entities having events during training period, following~\cite{RE-GCN}. Table~\ref{tb:data} shows the dataset statistics, including 2 source languages and 6 target languages. We purposefully choose 6 different target languages with diverse characteristics in term of the TKG size, which can evaluate \model from different data granularity. It is worth noting that to simulate the scarcity issue in target TKGs, the training quadruples presented in Table~\ref{tb:data} are randomly selected from original TKGs, with random ratio $20\%$.

\subsection{Baselines}
\label{ap:baseline}
We describe the baselines utilized in the experiments in detail:
\begin{itemize}[leftmargin = 15pt]
    \item \textbf{TransE}~\cite{TransE} is a translation-based embedding model, where both entities and relations are represented as vectors in the latent space. The relation is utilized as a translation operation between the subject and the object entity;
    \item \textbf{TransR}~\cite{TransR}  advances TransE by optimizing modeling of n-n relations, where each entity embedding can be projected to hyperplanes defined by relations;
    \item \textbf{DistMult}~\cite{DistMult} is a general framework with bilinear objective for multi-relational learning that unifies most multi-relational embedding models;
    \item \textbf{RotatE}~\cite{RotatE} represents entities as complex vectors and relations as rotation operations in a complex vector space;
    \item \textbf{TA-DistMult}~\cite{TA-DistMult} is a temporal knowledge graph reansoing method aiming at predicting missing events in history. We utilize it for predicting future events; 
    \item \textbf{RE-NET}~\cite{Renet} is a generative model to predict future facts on temporal knowledge graphs, which employs a recurrent neural network to model the entity evolution, and utilizes a neighborhood aggregator to consider the connection of facts at the same time intervals;
    \item \textbf{RE-GCN}~\cite{RE-GCN} learns the temporal representations of both entities and relations by modeling the KG sequence recurrently;
    \item \textbf{KEnS}~\cite{KEnS} starts to directly improve KGR performance on static KGs given a set of seed alignment, and proposes an ensemble-based approach for the task;
    \item \textbf{AlignKGC}~\cite{AlignKGC} jointly optimizes entity alignment loss and knowledge graph reasoning loss to improve the performance;
    \item \textbf{SS-AGA}~\cite{SS-AGA} views alignments as new edge type and employ a relation-aware GNN with learnable attention weight to model the influence of the aligned entities.
\end{itemize}

\subsection{Reproducibility}
\label{ap:implementation}
\subsubsection{Baseline Setup}
 For static knowledge graph reasoning methods, i.e., TransE, TransR, DistMult, and RotatE, we ignore all time information in quadruples, and view temporal knowledge graphs as static, cumulative ones. For static/temporal KG embedding methods, we merge source graph and target graph by adding one new type of relation (alignment), as they do not explicitly model cross-lingual entity alignment. For multilingual baselines, we train them on 1-to-1 knowledge transferring (instead of the original setting) for fair comparison. For static baselines, we utilize the static embeddings for predictions in all time steps. For fair comparisons, we keep the dimension of all embeddings as $128$, we feed pre-trained TransE embeddings on the merge graph including both source and target TKGs to those that require initial entity/relation embeddings. We tune learning rate of baselines based on {\em MRR} on validation set, and we train all baseline models and \model on same GPUs (Nvidia A100) and CPUs (Intel(R) Xeon(R) Platinum 8275CL).
 
\subsubsection{\model Setup}
We first utilize the source TKG to train the teacher representation module. Then we initialize the student module with the parameters of the teacher. During the training procedure, we first optimize the objective without generating pseudo data in the first 10 epoch. After that, we start to generate high-quality pseudo data. For the generation in each epoch, we gradually increase the amount of pseudo alignments from $10\%$ to $40\%$, and transfer all temporal events that meet the requirement. During evaluation, we tune hyperparameters based on {\em MRR} on validation set, and report the performance on the test set. Next, we report the choices of hyperparameters. For model training, we utilize Adam optimizer, and set maximum number of epochs as $50$. We set batch size as $256$, the dimension of all embeddings as $128$, and dropout rate as $0.5$. For the sake of efficiency,  we set number of temporal neighbors $b$ as $8$, and employ $1$ neighborhood aggregation layer in temporal encoder. For TKG reasoning, we set negative sampling factor as 10. For entity alignment, we set negative sampling factor as 50. For temporal generation process, We divide time span into $4$ time intervals. For model training, we mainly tune margin value $\lambda_1$, $\lambda_2$ in score functions in range $\{0.1,0.2,0.3,0.4,0.5,0.6,0.7,0.8,0.9\}$, learning rate  in range $\{0.02,0.01,0.005,0.001,0.0005\}$. 

\subsubsection{Efficiency Comparison}
\label{ap:time}
To demonstrate the efficiency of \model framework, we train \model and baseline models from scratch on both target language and source language, and compare the training time. We train all baseline models and \model on same GPUs (Nvidia A100) and CPUs (Intel(R) Xeon(R) Platinum 8275CL). Figure~\ref{fig:time} shows that \model significantly outperforms baseline models with reasonable training time. Notably, we include the pseudo data generation time. Compared with slow temporal models {\em RE-NET, RE-GCN} for knowledge graph reasoning, \model is more efficient because our temporal encoder can learn temporal entity embeddings via sampled temporal neighbors at each time without using RNNs.

\end{document}